\documentclass{article}

\PassOptionsToPackage{numbers, compress}{natbib}
\usepackage[preprint]{neurips_2025}

\usepackage[utf8]{inputenc} 
\usepackage[T1]{fontenc}    
\usepackage{hyperref}       
\usepackage{url}            
\usepackage{booktabs}       
\usepackage{amsfonts}       
\usepackage{nicefrac}       
\usepackage{microtype}      
\usepackage{graphicx}
\usepackage{natbib}
\usepackage{doi}

\usepackage{amsmath}
\usepackage{amsthm}
\usepackage[switch]{lineno}
\usepackage[table,usenames,dvipsnames]{xcolor}
\usepackage[small]{caption}
\usepackage{subcaption}
\usepackage{amssymb}
\usepackage{mathtools}
\usepackage{bm}
\usepackage{xspace}
\usepackage{url}
\usepackage{comment}

\usepackage{microtype}      
\usepackage{algorithm}
\usepackage[noend]{algorithmic}
\usepackage[inline]{enumitem}
\usepackage{comment}
\usepackage{threeparttable}
\usepackage{tikz}
\usepackage{multirow}
\usepackage{float}
\usepackage{placeins}
\usepackage{balance}

\title{\mac: An Efficient Gradient Preconditioning using
Mean Activation Approximated Curvature}

\author{%
  Hyunseok Seung$^{1}$ \quad
  Jaewoo Lee$^{2}$ \quad
  Hyunsuk Ko$^{3}$ \\
  $^{1}$University of Wisconsin -- Madison \quad
  $^{2}$University of Georgia \quad
  $^{3}$Hanyang University \\
  \texttt{hseung2@wisc.edu}, \ \texttt{jaewoo.lee@uga.edu}, \ \texttt{hyunsuk@hanyang.ac.kr}
}
\let\en=\ensuremath

\DeclareMathOperator{\E}{\mathbb{E}}          

\DeclarePairedDelimiter{\norm}{\lVert}{\rVert}
\DeclarePairedDelimiter{\abs}{\lvert}{\rvert}

\DeclareMathOperator{\diag}{diag}      

\DeclareMathOperator{\vect}{vec}                

\renewcommand{\vec}[1]{\en{\bm{\mathrm{#1}}}}

\newcommand{\mat}[1]{\en{{\bm{\mathrm{#1}}}}}
\newcommand{\grad}[0]{\en{\nabla}}

\newcommand{\R}[0]{\mathbb{R}}

\renewcommand{\th}[0]{\textsuperscript{th}\xspace}

\newcommand{\smac}[0]{\textsc{SMAC}\xspace}
\newcommand{\mac}[0]{\textsc{MAC}\xspace}

\usepackage{pifont}
\newcommand{\xmark}{\ding{55}}%

\newcommand*\circled[1]{\tikz[baseline=(char.base)]{
            \node[shape=circle,draw,inner sep=0pt] (char) {#1};}}

\theoremstyle{plain}
\newtheorem{theorem}{Theorem}[section]
\newtheorem{proposition}[theorem]{Proposition}
\newtheorem{lemma}[theorem]{Lemma}

\newtheorem{remark}[theorem]{Remark}
\theoremstyle{definition}
\newtheorem{definition}[theorem]{Definition}
\newtheorem{assumption}[theorem]{Assumption}
\newtheorem{condition}[theorem]{Condition}

\begin{document}

\maketitle

\begin{abstract}
	Second-order optimization methods for training neural networks, such
as KFAC, exhibit superior convergence 
by utilizing curvature information of loss landscape. However, it comes at the
expense of high computational burden. In this work, we analyze the two
components that constitute the layer-wise Fisher information matrix
(FIM) used in KFAC: the Kronecker factors related to activations and
pre-activation gradients.
Based on empirical observations on their eigenspectra, we propose
efficient approximations for them, resulting in a computationally
efficient optimization method called \mac. 
To the best of our knowledge, 
\mac is the first algorithm to apply the Kronecker factorization to the
FIM of attention layers used in transformers
and explicitly integrate attention scores into the preconditioning.
We also study the convergence property of \mac on nonlinear 
neural networks and provide two conditions under which it converges to
global minima. Our extensive evaluations on various network
architectures and datasets show that 
the proposed method outperforms
KFAC and other state-of-the-art methods in
terms of accuracy, end-to-end training time, and memory usage.


\end{abstract}

This is the extended version of the paper accepted to the IEEE International Conference on Data Mining (ICDM-2025), © IEEE. Code is available at \href{https://github.com/hseung88/mac}{https://github.com/hseung88/mac}.

\section{Introduction}
\label{sec:introduction}
Second-order methods have demonstrated faster convergence rates
than first-order methods in deep learning tasks~\cite{Amari1998NaturalGW,Gupta2018ShampooPS,Yao2020ADAHESSIANAA,liu2024sophia,Liu2024ALN}.
However, incorporating second-order information into optimization
introduces new challenges: high \emph{computational cost} and numerical
\emph{instability}. 
To address these challenges, KFAC~\cite{martens2015optimizing} and its
variants~\cite{Ba2016DistributedSO,Goldfarb2020PracticalQM,frantar2021mfac,Zhang2023EvaPS}
employ the layer-wise approximation and further decompose FIM
$\mat{F}^{(l)}$ for 
layer $l$ into the product of two smaller matrices:
$\mat{F}^{(l)} \approx \mat{A}^{(l)}\otimes \mat{P}^{(l)}$, where
$\mat{A}^{(l)}$ and $\mat{P}^{(l)}$ capture curvature information
associated with the input and output of the layer, respectively.
While 
this greatly reduces the computational cost,
the precise impact of 
Kronecker factors (KFs) $\mat{A}$ and $\mat{P}$ on the performance of
optimizer is not well understood~\cite{Benzing2022GradientDO}, and
second-order methods are not widely adopted in practice.

First, for many modern network architectures, the KFs, $\mat{A}$ and
$\mat{P}$, are still too large to store and compute the inverse. Despite
the fast convergence, second-order methods are often slower than first-order
methods in terms of wall-clock time; for example,
Shampoo~\cite{Gupta2018ShampooPS} requires 214.9 seconds,
AdaHessian~\cite{Yao2020ADAHESSIANAA} takes 45.8 seconds, and KFAC
averages 21.8 seconds per epoch when training ResNet-110 on the
CIFAR-10 dataset, whereas SGD completes an epoch in just 8.9 seconds.
Second, in our experiments, we observed that the variants of KFAC
(including KFAC) occasionally become numerically unstable and crash while
computing the inverse.
Third, existing KFAC methods have been limited to MLP and CNN
architectures, despite the widespread adoption of transformer
architectures in recent years. Transformers introduce unique
challenges for curvature approximation due to the complex structure of
self-attention layers, and a rigorous derivation of the KFAC FIM for
these layers has yet to be attempted. 

In this work, we investigate the structural properties of KFs and,
guided by our analysis, design  
an optimization method that is both time- and memory-efficient.
Specifically, in Section~\ref{subsec:eigenspectrum}, we
perform eigenanalysis on $\mat{A}$ and $\mat{P}$ and make the following
important observations.
\begin{enumerate*}[label=(\roman*)]
\item Both $\mat{A}$ and $\mat{P}$ have very few large eigenvalues and
the rest of them are much smaller in magnitude and take almost the
same value, meaning $\mat{A}$ and 
$\mat{P}$ are likely to be rank-deficient and hence their inversion
could be numerically unstable.
\item The eigenvalues of $\mat{A}$ are order of magnitude larger than
those of $\mat{P}$. Since the eigenvalues of $\mat{F}$ are products
of eigenvalues of $\mat{A}$ and $\mat{P}$, the magnitude of
eigenvalues of $\mat{F}$ are largely determined by that of $\mat{A}$
while the eigengap is by that of $\mat{P}$ (see
Section~\ref{subsec:eigenspectrum} for details). We conjecture this
might be the reason why KFAC with $\mat{P}$ replaced by 
the identity matrix 
(equivalent to FOOF~\cite{Benzing2022GradientDO} optimizer) performs as good as the original KFAC.
\item The top eigenvector of $\mat{A}$ points the same direction as
the mean activations.
\end{enumerate*}

Building upon these observations, 
we propose an 
optimization method named \mac (\textsc{M}ean
\textsc{A}ctivation approximated \textsc{C}urvature). In the proposed
method, the factor $\mat{A}$ is approximated with the rank-1 matrix,
the outer product of mean of layer's activations. On the other hand, the factor
$\mat{P}$ is approximated with an identity matrix.
These approximations greatly reduce the computation by storing a small
vector instead of matrix and allow to 
exploit closed-form solution for the inversion of $\mat{F}$ and hence 
enhance the scalability. The key contributions of our work can be summarized as
follows: 
\begin{itemize}[leftmargin=*,noitemsep,topsep=0pt]
\item We perform an eigenanalysis on the KFs $\mat{A}$ and
  $\mat{P}$, and make important observations on their structural
  properties.
\item Exploiting the structural properties of KFs, we propose 
a novel optimization method that reduces end-to-end training time by up to 55.4\% compared to KFAC while keeping memory usage comparable to that of SGD.
\item 
To the best of our knowledge, 
\mac is the first algorithm to apply Kronecker factorization to the
FIM of self-attention layers and to incorporate attention scores in
the preconditioning. 
Our experimental results demonstrate that explicitly incorporating attention
scores in preconditioning can enhance the performance of vision
transformers~\cite{dosovitskiy2021an}, achieving up to a 3.6\% increase
in top-1 test accuracy on the ImageNet dataset compared to the original KFAC. 
\item We present a convergence analysis of \mac on nonlinear neural
  networks and identify two conditions under which it converges to
  global minima. 
\item We provide extensive experimental results on 
state-of-the-art KFAC variants
for image classification tasks, using various network architectures on CIFAR and ImageNet datasets.
Notably, \mac exhibits the fastest execution time among all KFAC variants, and achieves test accuracies that are either superior to or on par with the original KFAC. 
\end{itemize}


\section{Related Work}
\label{sec:related}
%
Many variants of KFAC have been proposed to reduce its computational
cost, each differing in how the KFs, $\mat{A}^{(l)}$ and
$\mat{P}^{(l)}$, are approximated.
%
To avoid matrix inversion, K-BFGS~\cite{Goldfarb2020PracticalQM}
directly estimates the inverse of the KFs using the BFGS update, while
MFAC~\cite{frantar2021mfac} estimates the inverse-Hessian vector
products. 
However, both K-BFGS and MFAC requires additional forward and backward
passes, leading to significantly higher computational
and memory demands.
The low-rank property of KFs has been exploited in prior work.
SKFAC~\cite{Tang2021SKFACTN} computes the inverse of both 
KFs using the Woodbury formula, leveraging their
low-rank nature. 
%
\cite{Koroko2022EfficientAO} proposed low-rank approximations
of KFs using singular value decomposition (SVD), 
%
but the prohibitive computational cost of SVD for large networks
remains a significant drawback. 
These approaches store KFs as matrices. However, for large networks,
this matrix representation becomes challenging to manage in terms of
both computation and storage. 
Recent KFAC variants further approximate the KFs to reduce the
computation. 
FOOF~\cite{Benzing2022GradientDO} replaces the
pre-activation gradient term $\mat{P}$ by an identity 
matrix based on the experimental observation that it does not
contribute to the efficacy of KFAC.
%
LNGD~\cite{Liu2024ALN} approximates $\mat{P}$ as a diagonal matrix,
$\diag(\mat{P})$, with an additional trace-matching constraint.
Despite the approximation, LNGD incurs significantly higher memory usage than KFAC
because it requires storing per-example activations and second moments
of pre-activation gradients rather than aggregated ones.
The work closest to ours is Eva~\cite{Zhang2023EvaPS}, as it also
applies rank-1 approximation to $\mat{A}$.
However,
Eva applies the rank-1 approximation not only to the activation term $\mat{A}$ but
also to the problematic pre-activation gradient term $\mat{P}$,  
yet the rationale for this choice is not well justified. 
In contrast, 
our algorithm applies the rank-1 approximation only to the
activation-related KF $\mat{A}$, motivated by 
the structural analysis of KFs, and maintain aggregated vectors as
state variables. Furthermore, Eva is not easily applicable to
transformers
and a naive implementation
would completely disregard the attention scores in the preconditioning.

\section{Preliminaries}
\label{sec:prelim}
\subsection{Notations}
For vectors, we use element-wise operations unless specified
otherwise. $(\vec{x})_{i}$ denotes the $i$\th coordinate of
$\vec{x}$. $\norm{\vec{x}}$ represents $L_2$ norm unless stated
otherwise, and $\norm{\mat{X}}_{F}$ denotes the Frobenius norm. We use $[N]$ to denote the set $\{1,2, \dots, N\}$ 
and
$\otimes$ to represent the Kronecker product. 
The vectorization operator, denoted by
$\vect(\cdot)$, takes $\mat{X}\in \R^{m\times n}$ as input and returns
a vector 
$\vect(\mat{X})\in \R^{mn}$ of length $mn$. That is, $\vect(\mat{X}) =
\begin{bmatrix}
\mat{X}_{*,1}^{\intercal} & \mat{X}_{*,2}^{\intercal} & \cdots & \mat{X}_{*,n}^{\intercal}
\end{bmatrix}^{\intercal}\,$, where $\mat{X}_{*,j}$ denotes the $j$\th
column of matrix $\mat{X}$. For a square matrix $\mat{A}\in \R^{n\times n}$, $\lambda(\mat{A})$
denotes the set of eigenvalues of $\mat{A}$, and for any matrix $\mat{A}$, $\sigma(\mat{A})$ denotes the set of its singular values. Both eigenvalues and singular values are assumed to be sorted in descending order.

\subsection{Setup for Architecture and Training}
Consider a network $f(\vec{x};\vec{\theta})$ consisting of $L$ layers
trained on a dataset $\mathcal{S} = \{(\vec{x}_i, y_i)\}_{i=1}^n$.
Let $\mat{W}^{(l)}\in \R^{m_l\times m_{l-1}}$ and $\vec{b}^{(l)}\in \R^{m_l}$ 
be the weight and bias of layer $l\in [L]$.
The forward step of $f$ is given by
\begin{align*}
  \vec{z}^{(l)}
  &= \mat{W}^{(l)} \vec{a}^{(l-1)} + \vec{b}^{(l)}\in \R^{m_{l}}\,, \qquad \vec{a}^{(l)} = \phi(\vec{z}^{(l)})  \in \R^{m_{l}}\,,
    \qquad \vec{a}^{(0)} = \vec{x}\,,\\
  \vec{\theta}^{(l)}
  &= [\vect(\mat{W}^{(l)})^\intercal, \vec{b}^{\intercal}]^{\intercal} \in
    \R^{m_{l}(m_{l-1}+ 1)}\,, \qquad  \vec{\theta} \in [(\vec{\theta}^{(1)})^{\intercal}, \ldots,
    (\vec{\theta}^{(L)})^{\intercal}]^{\intercal}\in \R^{p}\,,
\end{align*}
where $\vec{z}$, $\vec{a}$, and $\phi$ represent the pre-activations,
activations, and an activation function, respectively. For activations
and pre-activations, we use the subscript $i$ to indicate that the
statistic is associated with the $i$\th training example. For example,
$\vec{a}^{(l)}_i$ denotes the activation of layer $l$ when input is $\vec{x}_i$.
Given training examples $\mathcal{S}$, our goal is to solve the following
optimization problem:
\begin{equation} \label{eq:opt_prob}
  \min_{\vec{\theta} \in \R^{d}} \mathcal{L}(\vec{\theta}) :=
  \frac{1}{n}\sum_{i=1}^{n}{\ell\left(f(\vec{x}_i;\vec{\theta}), y_i
    \right)}, 
\end{equation}
where $\ell$ is a loss function.
To solve the problem~\eqref{eq:opt_prob}, the NGD~\cite{Amari1998NaturalGW} method
iteratively performs the following update:
\begin{equation} \label{eq:ngd_update}
    \vec{\theta}_{k+1} = \vec{\theta}_{k} - \eta_{k} \mat{F}_{k}^{-1}
    \nabla_{\vec{\theta}}{\mathcal{L}(\vec{\theta}_k)}, 
\end{equation}
where $\mat{F}_{k} = \E_{(\vec{x}, y)
  \sim p_{\vec{\theta}}(\vec{x}, y)}\left[\nabla_{\vec{\theta}}{\log p(\vec{x}, y
    | \vec{\theta}_{k})} \nabla_{\vec{\theta}}{\log p(\vec{x}, y |
    \vec{\theta}_{k})}^\intercal \right]$ denotes the exact FIM, $\eta_k$ is a learning rate,
$p_{\vec{\theta}}(\vec{x}, y)$
is the predictive distribution of underlying probabilistic model, and
$\nabla_{\vec{\theta}}{\mathcal{L}(\vec{\theta}_k)}$ is the gradient
of $\mathcal{L}$ with respect to the parameters
at iteration $k$.


\section{Mean Activation Approximated Curvature}
\label{sec:algo}
In this section, 
we derive \mac
for solving
the optimization problem~\eqref{eq:opt_prob} based on the structural
analysis of empirical FIM.

\begin{figure*}[tb]
    \centering
    \includegraphics[width=.9\textwidth]{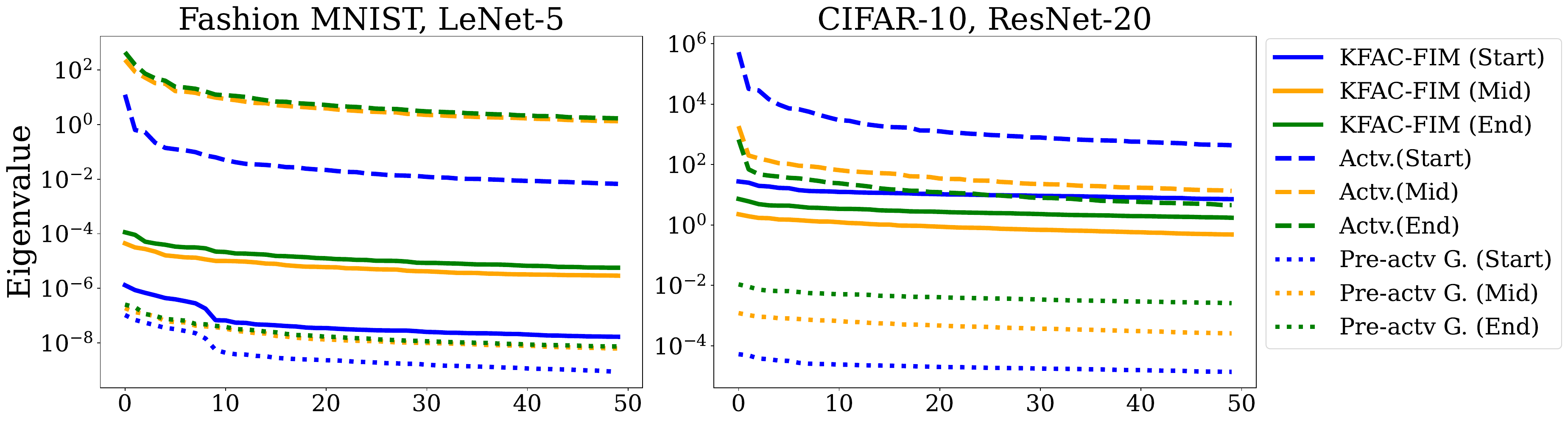}
    \includegraphics[width=.9\textwidth]{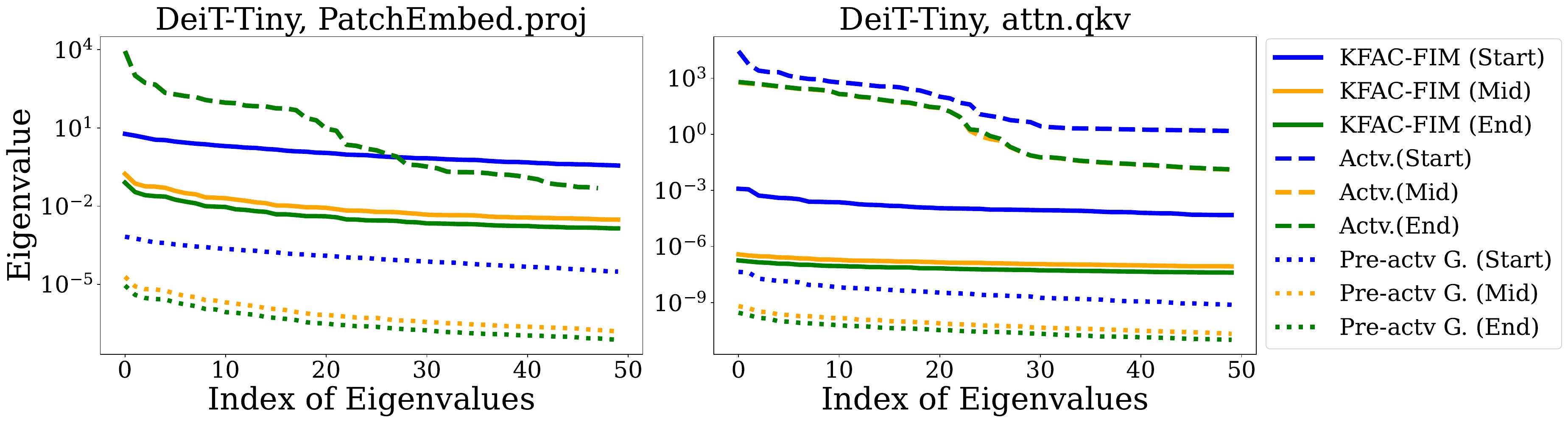}
    \caption{Top-50 eigenspectra of FIM, activation KF, and
      pre-activation gradient KF in KFAC were analyzed at the
      beginning, middle, and end of training. 
      (\textbf{Top}) A linear layer in LeNet-5 and a convolutional layer in ResNet-20. (\textbf{Bottom}) the patch embedding (convolutional) layer and an attention layer in DeiT-Tiny as representative examples.
      }
    \label{fig:eigenspectra}
\end{figure*}

\subsection{Eigenspectrum of FIM}
\label{subsec:eigenspectrum}
KFAC approximates $\mat{F}$ in Eq.~\eqref{eq:ngd_update} as a
block-diagonal matrix in which the $l$\th diagonal block is further
approximated as a Kronecker product of two smaller matrices:
\begin{align*}
  \mat{F}^{(l)} 
    &\approx
    \E\left[\vec{a}^{(l-1)} (\vec{a}^{(l-1)})^{\intercal}\right] \otimes
    \E\left[\vec{p}^{(l)}(\vec{p}^{(l)})^{\intercal}\right] \\
    &=
    \mat{A}^{(l)}\otimes \mat{P}^{(l)} =
    \mat{F}_{\text{KFAC}}^{(l)}\,,
\end{align*}
where $\vec{p}^{(l)} = \grad_{\vec{z}^{(l)}}\mathcal{L}$ is the gradient of
$\mathcal{L}$ w.r.t. the pre-activation $\vec{z}^{(l)}$.

\textbf{Structure of top eigenspace for $\mat{A}$.}
To understand
the structure of KFs, $\mat{A}$ and
$\mat{P}$, we perform an eigenanalysis.
Specifically,
we trained LeNet-5~\cite{lecun98lenet5} on Fashion
MNIST~\cite{xiao2017fashionmnist},
ResNet-20~\cite{he2016deep} on
CIFAR-10~\cite{krizhevsky2009learning}, and DeiT-Tiny~\cite{Touvron2020TrainingDI} on Tiny ImageNet~\cite{Le2015TinyIV}
using SGD (with momentum) and performed eigenvalue decomposition of
$\mat{A}$ and $\mat{P}$ to compare their top eigenvalues.
As depicted in Figure~\ref{fig:eigenspectra},
the approximated FIM
$\mat{F}_{\text{KFAC}}$ has very few large eigenvalues (notice the logscale on the
$y$-axis) and the remaining vast majority have 
significantly smaller magnitudes (close to 0), forming a long tail. The
same is observed from the eigenvalues of $\mat{A}$.
In contrast, the eigenvalues of $\mat{P}$
have significantly smaller magnitude compared to those of
$\mat{A}$. This is important because KFAC requires computing the inverse
$(\mat{P}^{(l)}+\rho\mat{I})^{-1}$, where $\rho$ is a damping factor, for preconditioning, which can
cause numerical instability, especially in low-precision
training~\cite{lin2023Structured,lin2023simplifying}. Indeed, we also
observed in our experiments that methods involving the computation of the inverse of $\mat{P}$
occasionally crash in large-scale training.
From the property of eigenvalues of Kronecker product
$\lambda(\mat{F}_{\text{KFAC}}) = \lambda(\mat{A}\otimes \mat{P}) = \{\mu\nu~:~\mu\in \lambda(\mat{A}),\, \nu\in
\lambda(\mat{P})\}$, we see that
the magnitude of eigenvalues in
$\lambda(\mat{F}_{\text{KFAC}})$ is largely determined by those in
$\lambda(\mat{A})$ while the eigengap of $\mat{F}_{\text{KFAC}}$ will be similar to
that of $\mat{P}$. 
This is 
evident in
Figure~\ref{fig:eigenspectra}. We see that the curves corresponding to
the distribution of $\lambda(\mat{F}_{\text{KFAC}})$ can be made to approximately
coincide with the curve corresponding to $\lambda(\mat{P})$ with some rescaling.
This not only shows 
the diminished role of $\mat{P}$ in characterizing
the optimization landscape but also raises questions about the extent to which it
contributes to the overall efficacy of KFAC.
\cite{Benzing2022GradientDO} empirically showed that removing the
factor $\mat{P}$ from KFAC does not degrade its performance.

\subsection{Curvature Approximation for MLPs and CNNs}
Based on the above observations, we propose 
an efficient approximation
of curvature information $\mat{F}_{\text{KFAC}}$:
\begin{align}
  \mat{F}_{\text{MAC}}^{(l)}
  &=\left(\bar{\vec{a}}^{(l-1)}(\bar{\vec{a}}^{(l-1)})^{\intercal}
  +\rho\mat{I}_{m_{l-1}}\right)\otimes \mat{I}_{m_l}\,, \label{eq:mac_approx}
\end{align}
where $\bar{\vec{a}}^{(l)} = \E[\vec{a}^{(l)}] =
\frac{1}{\abs{\mathcal{B}}}\sum_{i\in \mathcal{B}} \vec{a}_i^{(l)}$.
The method approximates
the activation
covariance matrix $\mat{A}$ with a rank-1
matrix,
$\bar{\vec{a}}^{(l)}(\bar{\vec{a}}^{(l)})^{\intercal}=\E[\vec{a}]\E[\vec{a}]^{\intercal}$.
Three important
remarks are in order. First, the covariance matrix is expressed
as $\E[\vec{a}\vec{a}^{\intercal}] = \E[\vec{a}]\E[\vec{a}]^{\intercal} +
\mat{\Sigma}_{\vec{a}}$, where $\mat{\Sigma}_{\vec{a}} = \E[(\vec{a} - \bar{\vec{a}})(\vec{a}-\bar{\vec{a}})^{\intercal}]$.
Intuitively, we have $\E[\vec{a}\vec{a}^{\intercal}] \approx
\E[\vec{a}]\E[\vec{a}]^{\intercal}$ if
$\norm{\mat{\Sigma}_{\vec{a}}}_{F}$ is small.
ReLU activation function is widely
used in many network architectures and it ensures that activation $\vec{a}$
is non-negative. As a result, in practice, we often observe the
(centered) covariance term $\mat{\Sigma}_{\vec{a}}$ has a small
magnitude.
We formalize this intuition and provide a sufficient condition for good approximation below.

\begin{figure*}[tb]
    \centering
    \includegraphics[width=.46\linewidth]{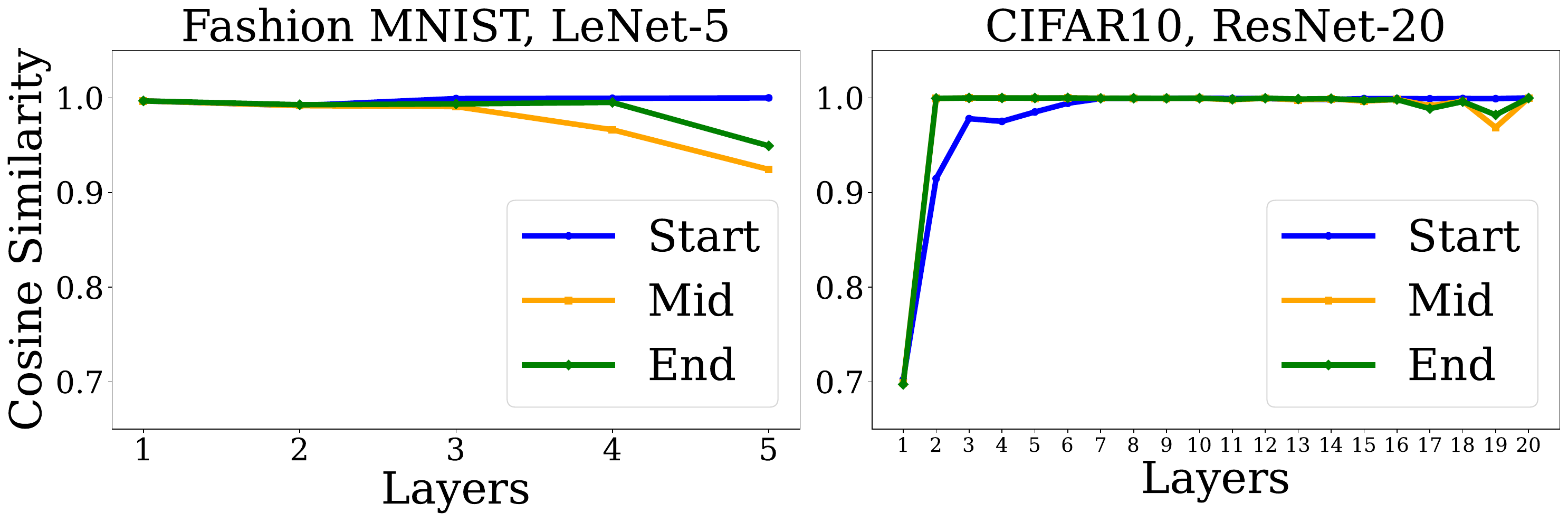}
    \includegraphics[width=.49\linewidth]{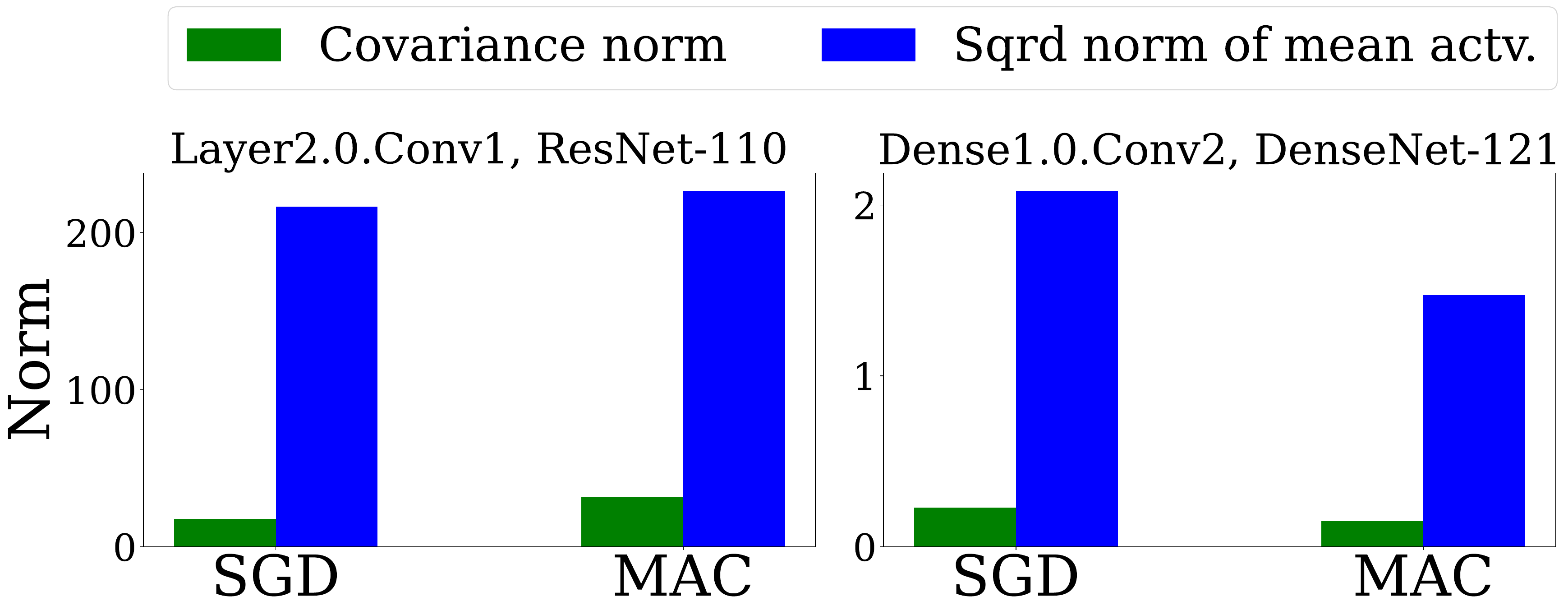}
    \caption{(\textbf{Left}) Cosine similarity between the top eigenvector of
      $\mat{A}$ and the mean activations per layer. (\textbf{Right}) Comparison of centered covariance norms with squared norms of mean activation using the CIFAR-100 dataset.}
    \label{fig:heatmap_cosine}
\end{figure*}

\begin{proposition} \label{prop:1}
   Let $\mat{X}$ be an $m \times n$ matrix with column-wise mean vector
   $\bar{\vec{x}} \in \R^n$. Define a perturbation matrix $\mat{E}$
   such that $\mat{X} = \vec{1}_m \bar{\vec{x}}^\intercal + \mat{E}$,
   where $\vec{1}_m$ is an $m$-dimensional column vector of ones. For some small $\epsilon > 0$, if the Frobenius norm of $\mathbf{E}$ satisfies $\|\mathbf{E}\|_F \leq
    \sqrt{m}\norm{\bar{\mathbf{x}}}(\sqrt{1+\epsilon} -1)$, then $\mat{X}^\intercal \mat{X} \approx m\bar{\vec{x}}\bar{\vec{x}}^\intercal$. 
\end{proposition}

Note that, in our setup, the matrix $\mat{E}$ represents the
deviations $\vec{a} - \E[\vec{a}]$ and $\E[\mat{E}^\intercal \mat{E}]
= \mat{\Sigma}_{\vec{a}}$. The right-hand side of Figure~\ref{fig:heatmap_cosine} provides an empirical evidence that
$\norm{\mat{\Sigma}_{\vec{a}}}_{F}$ is considerably smaller than
$\norm{\E[\vec{a}]}^2$.
Second, we now show that the eigenvector of $\mat{A}$ corresponding to the largest
eigenvalue well aligns with $\bar{\vec{a}} = \E[\vec{a}]$.
Notice that this is the only eigenvector (corresponding to non-zero
eigenvalue) of rank-1 matrix 
$\bar{\vec{a}}\bar{\vec{a}}^{\intercal}$.
From the definition of $\mat{\Sigma}_{\vec{a}}$, we have
$\mat{A}=\frac{1}{n}\sum_{i=1}^{n}\vec{a}_i\vec{a}_i^{\intercal} =
\bar{\vec{a}}\bar{\vec{a}}^{\intercal} + \mat{\Sigma}_{\vec{a}} =
\mat{M}+ \mat{\Sigma}_{\vec{a}}$,
where $\mat{M} = \bar{\vec{a}}\bar{\vec{a}}^{\intercal}$.
Let $\lambda_1, \ldots, \lambda_n$ and $\hat{\lambda}_1, \ldots,
\hat{\lambda}_n$ be the eigenvalues of $\mat{M}$ and $\mat{A} =
\mat{M}+\mat{\Sigma}_{\vec{a}}$. Furthermore, let $\vec{v}_1, \ldots, \vec{v}_n$
and $\hat{\vec{v}}_1, \ldots, \hat{\vec{v}}_n$ denote the eigenvectors
of $\mat{M}$ and $\mat{A}$, respectively. 
Using the Davis-Kahan theorem~\cite{Yu2014AUV}, we can show that
$\norm{\vec{v}_1 -\hat{\vec{v}}_1}_2 \leq
  \frac{2\sqrt{2}\norm{\mat{\Sigma}_{\vec{a}}}_F}{\norm{\bar{\vec{a}}}^{2}}$.
Note that 
$\vec{v}_1 = \bar{\vec{a}}/\norm{\bar{\vec{a}}}$ and 
the term on the RHS can be sufficiently small in practice (see
the right-hand graph of Figure~\ref{fig:heatmap_cosine}).
To empirically verify this,
we trained two different
architectures on two distinct datasets and computed the cosine
similarity between top eigenvector of $\mat{A}$ and
$\bar{\vec{a}}$. As shown on the left of Figure~\ref{fig:heatmap_cosine}, their
cosine similarity approaches nearly 1 across layers, meaning that the top
eigenvector of $\mat{A}$ points the same direction as $\bar{\vec{a}}$.
Third, the use of rank-1 approximation of $\mat{A}$ allows an efficient
inversion of $\mat{F}_{\mac}$ via the
Sherman-Morrison formula. KFAC requires storing $\mat{A}^{(l)}$ and
$\mat{P}^{(l)}$ for each layer $l$ and computing their inverses for
preconditioning. In general, the time complexity of computing inverse
of an $n\times n$ matrix is $\mathcal{O}(n^{3})$, and practical methods
periodically update $\mat{A}^{(l)}$ and $\mat{P}^{(l)}$ to amortize
the computation over iterations. Our rank-1 approximation
significantly improves
both time and memory complexity of algorithm as it stores
$\bar{\vec{a}}^{(l-1)}\in \R^{m_{l-1}}$ instead of $\mat{A}^{(l)}\in
\R^{m_{l-1}\times m_{l-1}}$ and uses the closed-form solution for inversion.

\subsection{Curvature Approximation for Transformers}
\label{subsec:curve_trf}
%
This section explains how the curvature information approximation
in~\eqref{eq:mac_approx} can be extended to the attention layers in
transformers. A direct
application of~\eqref{eq:mac_approx} to the attention layers results
in a form that entirely neglects the attention scores, which are the most
crucial component of the self-attention mechanism.
Through an in-depth analysis of the backpropagation step for attention
layers, we derive an efficient approximation of the FIM.

Let $\mat{X} \in \R^{N \times d}$ denote a sequence of $N$ tokens
embedded in a $d$-dimensional space. In the self-attention mechanism,
the input is linearly projected to obtain the query, key, and value
matrices: 
\[
\mat{Q}_h = \mat{X}\mat{W}_{h,q}, \quad \mat{K}_h =
\mat{X}\mat{W}_{h,k}, \quad \mat{V}_h = \mat{X}\mat{W}_{h,v} 
\]
for $h = 1,2,\dots,H$, where 
%
$h$ indexes the attention head, $H$ is the total number of attention heads,
$\mat{W}_{h,q},\,\mat{W}_{h,k},\,\mat{W}_{h,v} \in \R^{d\times d_k}$,
and $d_k = d/H$. In vision transformers (ViTs), these projections are
often implemented as a single linear mapping $\mat{Z} =
\mat{X}\mat{W}_{qkv} \in \R^{N\times 3d}\,,$ which is then partitioned
as $\mat{Q} = \mat{Z}_{:,1:d},\ \mat{K} = \mat{Z}_{:,d+1:2d},\ \mat{V}
= \mat{Z}_{:,2d+1:3d}$. For each attention head, the scaled dot
product is computed as $\mat{R} = \mat{Q}\mat{K}^\intercal /
\sqrt{d_k} \in \R^{N \times N}\,$ and the attention weights are
obtained by $\mat{T} = \text{softmax}(\mat{R}) \in \R^{N \times
  N}$. The head output is then given by $\mat{H}_h =
\mat{T}\mat{V}_h\,,$ and the final output of the attention layer is
formed by concatenating the outputs from all heads and applying a
final linear projection.  

During backpropagation, we have
\begin{align*}
  \frac{\partial \mathcal{L}}{\partial \mat{W}_q}
  &=\frac{\partial \mathcal{L}}{\partial \mat{R}} \frac{\partial
    \mat{R}}{\partial \mat{Q}} \frac{\partial \mat{Q}}{\partial \mat{W}_q} 
    = \mat{X}^\intercal\, \frac{\partial \mathcal{L}}{\partial
    \mat{R}}\, \mat{K}    
    = \mat{X}^\intercal \Delta_{\mat{R}}\, \mat{K}\,, \\
  \frac{\partial \mathcal{L}}{\partial \mat{W}_k}
  &=\frac{\partial \mathcal{L}}{\partial \mat{R}}
    \frac{\partial \mat{R}}{\partial \mat{K}}
    \frac{\partial \mat{K}}{\partial \mat{W}_k}
    =\mat{X}^\intercal\, \Delta_{\mat{R}}^\intercal\, \mat{Q}\,, \\
  \frac{\partial \mathcal{L}}{\partial \mat{W}_v}
  &=\frac{\partial \mathcal{L}}{\partial \mat{H}}
    \frac{\partial \mat{H}}{\partial \mat{V}}
    \frac{\partial \mat{V}}{\partial \mat{W}_v}
    = \mat{X}^\intercal\, \mat{T}^\intercal\, \Delta_{\mat{H}}\,,
\end{align*} 
where \(\Delta_{\mat{R}} = \partial \mathcal{L}/\partial \mat{R}\) and \(\Delta_{\mat{H}} = \partial \mathcal{L}/\partial \mat{H}\).
\begin{figure*}[tb]
    \centering
    \includegraphics[width=.98\linewidth]{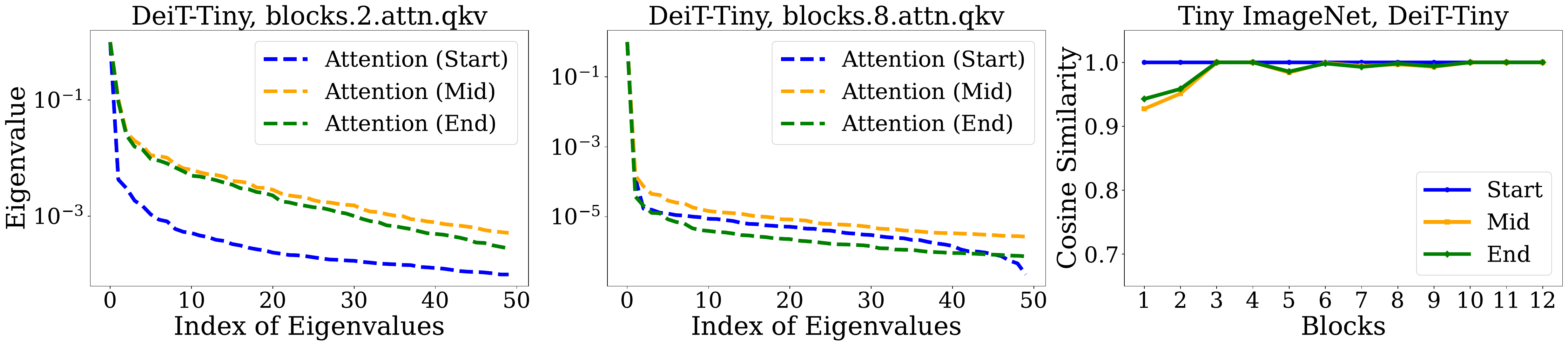}
    \caption{Trained DeiT-Tiny on Tiny ImageNet. (\textbf{Left, Center}) Eigenspectra of attention scores $\mat{T}$ from two distinct blocks as representative cases. (\textbf{Right}) Cosine similarity between the top eigenvector of $\mat{T}$ and the mean attention per block.}
    \label{fig:attn_deit}
\end{figure*}
The empirical FIMs corresponding to the query, key, and value components are then computed as
\begin{align}
    \mat{F}_q &= \E\left[\vect(\mat{X}^\intercal\Delta_{\mat{R}}\mat{K})\vect(\mat{X}^\intercal\Delta_{\mat{R}}\mat{K})^\intercal\right] \nonumber\\
    &= \E\left[(\mat{K}^\intercal \otimes \mat{X}^\intercal)\vect(\Delta_{\mat{R}})\vect(\Delta_{\mat{R}})^\intercal (\mat{K} \otimes \mat{X})\right]\,, \label{eq:fim_q}\\ 
    \mat{F}_k &= \E\left[\vect(\mat{X}^\intercal\Delta_{\mat{R}}\mat{Q})\vect(\mat{X}^\intercal\Delta_{\mat{R}}\mat{Q})^\intercal\right] \nonumber\\
    &= \E\left[(\mat{Q}^\intercal \otimes \mat{X}^\intercal)\vect(\Delta_{\mat{R}})\vect(\Delta_{\mat{R}})^\intercal (\mat{Q} \otimes \mat{X})\right]\,, \label{eq:fim_k}\\
    \mat{F}_v &= \E\left[\vect(\mat{X}^\intercal\mat{T}^\intercal \Delta_{\mat{H}})\vect(\mat{X}^\intercal\mat{T}^\intercal \Delta_{\mat{H}})^\intercal\right] \nonumber\\
    &= \E\left[(\Delta_H^\intercal \otimes \mat{X}^\intercal)\vect(\mat{T})\vect(\mat{T})^\intercal (\Delta_H \otimes \mat{X})\right]\,. \label{eq:fim_v}
\end{align}
As discussed in Section~\ref{subsec:eigenspectrum} and
%
demonstrated in Figure~\ref{fig:eigenspectra}, 
%
the eigenvalue distribution of $\mat{F}_{\text{KFAC}}$ is primarily determined by $\mat{A}$, while the contribution from $\mat{P}$ appears almost uniformly scaled. Therefore, we approximate the pre-activation gradient covariance $\vect(\Delta_{\mat{R}})\vect(\Delta_{\mat{R}})^\intercal$ in \eqref{eq:fim_q} and \eqref{eq:fim_k} with an identity matrix. 
%
Furthermore,
empirical analysis shows that the attention score matrix $\mat{T}$ is
rank-deficient. We therefore approximate
$\vect(\mat{T})\vect(\mat{T})^\intercal =
\vect(\bar{\vec{t}}\bar{\vec{t}}^\intercal)\vect(\bar{\vec{t}}\bar{\vec{t}}^\intercal)^\intercal
= (\bar{\vec{t}} \otimes \bar{\vec{t}})(\bar{\vec{t}} \otimes
\bar{\vec{t}})^\intercal$, where $\bar{\vec{t}} \in \R^{N}$ is the
column-wise average of \(\mat{T}\), representing the average attention
distribution over tokens. Figure~\ref{fig:attn_deit} confirms that
$\mat{T}$ exhibits a dominant top eigenvalue and that the cosine
similarity between its top eigenvector and $\bar{\vec{t}}$ is nearly
1, justifying the rank-1 approximation $\mat{T} \approx
\bar{\vec{t}}\bar{\vec{t}}^\intercal$.  

Based on these observations, we approximate the FIMs as follows:
\begin{align*}
    \mat{F}_q &\approx \E[\mat{K}^\intercal \mat{K}] \otimes \E[\mat{X}^\intercal \mat{X}]\,, \\
    \mat{F}_k &\approx \E[\mat{Q}^\intercal \mat{Q}] \otimes \E[\mat{X}^\intercal\mat{X}]\,, \\
    \mat{F}_v &\approx \E\left[\Delta_H^\intercal \bar{\vec{t}}\bar{\vec{t}}^\intercal \Delta_H\right] \otimes \E\left[\mat{X}^\intercal \bar{\vec{t}}\bar{\vec{t}}^\intercal \mat{X}\right]\,,
\end{align*}
Empirically,
we observed that the covariance terms
$\E[\mat{K}^\intercal \mat{K}]$, $\E[\mat{Q}^\intercal \mat{Q}]$, and
$\E\left[\Delta_H^\intercal \bar{\vec{t}}\bar{\vec{t}}^\intercal
  \Delta_H\right]$ add little benefit to \mac's performance while
incurring extra computational cost. Consistent with the motivation
behind \mac, we omit these terms and
%
propose the following efficient approximation of curvature information
for attention layers:
\begin{align*}
    \mat{F}_{\mac, qkv} &= \diag\left(\mat{F}_{\mac, q},\, \mat{F}_{\mac, k},\, \mat{F}_{\mac, v}\right)\,, \ \text{where}\\
    \mat{F}_{\mac, q} &= \mat{I}_d \otimes \E[\vec{x}]\E[\vec{x}]^\intercal\,, \\
    \mat{F}_{\mac, k} &= \mat{I}_d \otimes \E[\vec{x}]\E[\vec{x}]^\intercal\,, \\
    \mat{F}_{\mac, v} &= \mat{I}_d \otimes \E[\mat{X}^\intercal \bar{\vec{t}}]\E[\mat{X}^\intercal \bar{\vec{t}}]^\intercal\,,
\end{align*}
where $\vec{x} \in \R^d$ denotes a single input token and $\mat{X}^\intercal \bar{\vec{t}} \in \R^d$.

%
Our derivation shows that capturing the curvature information for
attention layers requires extracting information from intermediate
quantities: query, key, and attention scores. Specifically,
unlike the MAC approximation for fully-connected layers presented
in~\eqref{eq:mac_approx}, the FIM $\mat{F}_{\mac, v}$ for value
reweights the activations (i.e., input sequence of tokens) using the (mean) attention 
scores as weights, explicitly incorporating the attentions into
the preconditioner. In addition, the rank-1 approximation in our
derivation enables an efficient inversion of the FIM using the
Sherman-Morrison formula. 

\subsection{Decoupled and Adaptive Damping}
Using the Sherman-Morrison formula and the property
$(\mat{B}^{\intercal}\otimes \mat{A})\vect(\mat{X}) =
\vect(\mat{A}\mat{X}\mat{B})$, the update equation for \mac in matrix 
form is given by 
%
\begin{align*}
  \vec{\theta}
  &\gets \vec{\theta} - 
    \frac{\eta}{\rho}\mat{G}\left(\mat{I}_{m_{l-1}} -
    \frac{\E[\vec{a}]\E[\vec{a}]^\intercal}{\rho + \|\E[\vec{a}]\|^2}\right) \\
    &= \underbrace{\vec{\theta}-\frac{\eta}{\rho}\mat{G}}_{\circled{1}}+
    \underbrace{\frac{\eta}{\rho}\cdot
    \frac{\mat{G}\E[\vec{a}]\E[\vec{a}]^\intercal}{(\rho +
    \|\E[\vec{a}]\|^2)}}_{\circled{2}}\,,
\end{align*}
where $\mat{G} =
\vect^{-1}\left(\grad_{\vec{\theta}}{\mathcal{L}}\right)\in
\R^{m_{l}\times m_{l-1}}$ and $\vec{\theta} \in
\R^{m_{l}\times m_{l-1}}$.
The update equation can be viewed as the standard SGD update
\circled{1}, plus the correction term \circled{2}, which projects
the row vectors of $\mat{G}$ onto the line defined by the mean activation
$\E[\vec{a}]$, with $\rho$ interpreted as 
a regularization coefficient. Although $\rho$ was introduced to control
the strength of damping, we observe that it also influences the
magnitude of \mac update. Specifically, when $\rho$ is small, the
effective step size $\eta/\rho$ becomes excessively large, leading to
unstable updates.
Conversely, when $\rho$ is large, the effective step size becomes too
small, slowing convergence and degrading the performance. 
To enhance
stability and prevent performance degradation, we decouple $\rho$ from
$\eta$, 
leading to a more stable update:  
\begin{equation} \label{eq:decoupled_damp}
  \vec{\theta} \gets \vec{\theta} - \eta
  \mat{G}\left(\mat{I}_{m_{l-1}} -
      \frac{\E[\vec{a}]\E[\vec{a}]^\intercal}{\rho +
        \|\E[\vec{a}]\|^2}\right)\,, 
\end{equation}
where $\eta$ is redefined to absorb $\rho$ in the effective step
size. In this formulation, $\rho$ functions solely as a
%
regularization coefficient, ensuring that the update remains
well-scaled and stable. 

Additionally, to maintain the accuracy of our curvature approximation,
we introduce an adaptive damping strategy.
We dynamically adjusts $\rho$ during run time such that the trace of
the approximated covariance matches that of the true covariance. That is,
$\text{trace}(\E[\vec{a}\vec{a}^\intercal]) = \text{trace}(\E[\vec{a}]\E[\vec{a}]^\intercal + \rho \mat{I}_{m_{l-1}})=  \|\E[\vec{a}]\|^2 + \rho m_{l-1}$.
Solving for $\rho$ gives 
\[
   \rho = \frac{\text{trace}(\E[\vec{a}\vec{a}^\intercal]) - \|\E[\vec{a}]\|^2}{m_{l-1}}\,.
\]

\subsection{Algorithm}
The pseudocode of \mac is presented in
Algorithm~\ref{alg:pseudocode}. Since, for the first layer ($l=1$),
$\vec{a}_i^{(0)}=\vec{x}_i$ is the input data example and fixed
throughout the training, we can optionally 
pre-compute $\bar{\vec{a}}^{(0)}=(1/n)\sum_{i=1}^{n}\vec{x}_i$ before
starting the training and skip updating preconditioner for the first
layer during training (see Line~\ref{alg:firstlayer}). We empirically
observed that computing $\bar{\vec{a}}^{(0)}$ over mini-batches during
runtime incurs negligible increase in execution time.
At each iteration $k$, \mac estimates
$\bar{\vec{a}}^{(l)}$ using the examples in the current mini-batch $\mathcal{B}$
and update the maintained statistic using an exponential moving average (EMA) as shown in Line~\ref{alg:attn_a_update} for attention value projection layers and in Line~\ref{alg:a_update} for fully-connected and convolutional layers.
Line~\ref{alg:a_inv} applies Sherman-Morrison formula to compute the inverse of damped rank-1 approximation on $\mat{A}$.

\renewcommand{\algorithmiccomment}[1]{$\triangleright$ #1} 
\begin{algorithm}[t]
\caption{\mac}
\label{alg:pseudocode}
\textbf{Require}: Learning rate $\eta_k$, Momentum $\beta_{1}$, EMA $\beta_{2}$, Damping $\rho$, Curvature update frequency $\tau_{\text{cov}}$, Inverse update frequency $\tau_{\text{inv}}$\\
\textbf{Initialize}:
\makebox[\linewidth][l]{$\vec{\theta}_0,\ \widetilde{\vec{a}}_0\!=\!\vec{0},\ (\widehat{\mat{A}}^{(l)})^{-1}\!=\!\mat{I},\ k_\tau\!=\!0$}
\begin{algorithmic}[1]
\STATE (optional)  $(\mat{A}^{(0)})^{-1} = \left(\mat{I}
  - \frac{\bar{\vec{a}}^{(0)} (\bar{\vec{a}}^{(0)})^{\intercal}}{\rho +
    \norm{\bar{\vec{a}}^{(0)}}^{2}}\right)$, where $\bar{\vec{a}}^{(0)} =
\frac{1}{n}\sum_{i=1}^{n}\vec{x}_i$ \label{alg:firstlayer}
\FOR{$k$ = 1, 2, 3, \dots}
\STATE $\mat{G}_{k} \gets \frac{1}{|\mathcal{B}|}\sum_{i\in
  \mathcal{B}}\grad{\ell}(f(\vec{x}_i; \vec{\theta}_k), y_i)$
\FOR{$l$ = 1, 2, \dots, L}
\IF{$(k \mod \tau_{\text{cov}}) = 0$}
\IF{$\vec{\theta}^{(l)} = \vec{\theta}_{v}^{(l)}$}
\STATE $\widetilde{\vec{a}}_{k}^{(l)} \gets
\beta_{2}\widetilde{\vec{a}}_{k-1}^{(l)} +
(1-\beta_{2})(\mat{X}_k^{(l)})^\intercal\bar{\vec{t}}_k^{(l)}$ \label{alg:attn_a_update}
\STATE \COMMENT{for $\mat{W}_v$ in attention layers} 
\ELSE
\STATE $\widetilde{\vec{a}}_{k}^{(l)} \gets \beta_{2}\widetilde{\vec{a}}_{k-1}^{(l)} +
(1-\beta_{2})\bar{\vec{a}}_{k}^{(l-1)}$ \label{alg:a_update}
\ENDIF
\STATE $k_{\tau} \gets k_{\tau} + 1$
\ENDIF
\IF{$(k \mod \tau_{\text{inv}}) = 0$}
\STATE $\widehat{\vec{a}}_{k}^{(l)} \gets \widetilde{\vec{a}}_{k}^{(l)} \ / \
(1-\beta_{2}^{k_{\tau}})$
\STATE $(\widehat{\mat{A}}^{(l)})^{-1} \gets \left(\mat{I} - \frac{\widehat{\vec{a}}_{k}^{(l)}(\widehat{\vec{a}}_{k}^{(l)})^{\intercal}}{\rho + \|\widehat{\vec{a}}_{k}^{(l)}\|^{2}}\right)$ \label{alg:a_inv}
\STATE \COMMENT{decoupled damping in \eqref{eq:decoupled_damp}} 
\ENDIF
\STATE $\widehat{\mat{G}}_{k}^{(l)} \gets
\mat{G}_{k}^{(l)}(\widehat{\mat{A}}^{(l)})^{-1}$
\ENDFOR
\ENDFOR
\end{algorithmic}
\end{algorithm}

We compare the asymptotic time and memory costs of
preconditioning a layer with the weight matrix of size $(d_{\text{out}}
\times d_{\text{in}})$ in Table~\ref{tab:complexity}. It is evident
that 
\mac significantly reduces
computational and memory requirements compared to KFAC and its
variants. These reductions are achieved by simplifying the FIM
approximation process without compromising the optimization
performance, thereby offering a more efficient alternative for
leveraging second-order information in deep learning optimization
tasks. 

\begin{table*}[tb]
  \caption{Comparison of time and memory complexity for computing the inverse of preconditioner(s)}
  \label{tab:complexity}
  \centering
  \begin{tabular}{llll}
    \toprule
    Method & Preconditioner & Time & Memory \\
    \midrule
    \textsc{KFAC} & $\E[\vec{a}\vec{a}^\intercal] \otimes \E[\vec{p}\vec{p}^\intercal]$ & $\mathcal{O}(d_{\text{out}}^3) + \mathcal{O}(d_{\text{in}}^3)$ & $\mathcal{O}(d_{\text{out}}^2) + \mathcal{O}(d_{\text{in}}^2)$ \\
     \textsc{FOOF}~\cite{Benzing2022GradientDO} & $\E[\vec{a}\vec{a}^\intercal] \otimes \mat{I}$ & $\mathcal{O}(d_{\text{in}}^3)$ & $\mathcal{O}(d_{\text{in}}^2)$ \\
     \textsc{Eva}~\cite{Zhang2023EvaPS} & $\E[\vec{a}]\E[\vec{a}]^\intercal \otimes \E[\vec{p}]\E[\vec{p}]^\intercal$& $\mathcal{O}(d_{\text{out}}^2) + \mathcal{O}(d_{\text{in}}^2)$ & $\mathcal{O}(d_{\text{out}}) + \mathcal{O}(d_{\text{in}})$ \\
      \textsc{LNGD}~\cite{Liu2024ALN} & $\E[\|\vec{p}\|^2 \vec{a}\vec{a}^\intercal] \otimes \frac{\E[\|\vec{a}\|^2 \text{diag}(\vec{p}\vec{p}^\intercal)]}{\E[\|\vec{a}\|^2\|\vec{p}\|^2]}$ & $\mathcal{O}(d_{\text{out}}) + \mathcal{O}(d_{\text{in}}^3)$ & $\mathcal{O}(|\mathcal{B}| \cdot d_{\text{out}}) + \mathcal{O}(|\mathcal{B}| \cdot d_{\text{in}}^2)$ \\
      \midrule
     \mac & $\E[\vec{a}]\E[\vec{a}]^\intercal \otimes \mat{I}$ & $\mathcal{O}(d_{\text{in}}^2)$ & $\mathcal{O}(d_{\text{in}})$ \\
    \bottomrule
  \end{tabular}
\end{table*}

For hyperparameter settings, we recommend adopting the same learning
rate, momentum, and weight decay value as SGD. An ablation study of the remaining hyperparameters is provided in Section~\ref{apdx:ablation}.


\section{Convergence Analysis}
\label{sec:convergence}
In this section, we analyze the convergence of \mac on 
the 
2-layer ReLU network~\cite{du2018gradient}:
\begin{equation*} \label{eq:twolayernn}
     f(\vec{\theta}, \vec{x}) = \frac{1}{\sqrt{m}}\sum_{r=1}^{m}{q_{r} \phi(\vec{w}_{r}^\intercal \vec{x})}\,,
\end{equation*}
where $\mat{W} = [\vec{w}_1, \vec{w}_2, \dots, \vec{w}_m]^\intercal \in \R^{m\times d}$ is the weight matrix of the first layer, $\vec{\theta}=\vect{(\mat{W})}\in \R^{md}$, $\vec{q}=[q_1, q_2, \dots, q_m] \in \R^m$ is the weight of output layer initialized with $q_r \sim \text{Unif}(-1,1)$ and fixed during training. Given 
$\mathcal{S} = \{(\vec{x}_i, y_i)\}_{i=1}^n$, the error of network $f$ 
is measured using the squared loss function. The update equation of \mac is given by
\begin{equation*}
    \vec{\theta} \leftarrow  \vec{\theta} - \eta \mat{F}_{\mac}^{-1} \mat{J}(\vec{\theta})^\intercal (\vec{u}(\vec{\theta}) - \vec{y})
\end{equation*}
where $\mat{F}_{\mac} = (\bar{\vec{x}}\, \bar{\vec{x}}^\intercal + \rho \mat{I}_d) \otimes \mat{I}_m$, $\overline{\vec{x}} = \frac{1}{n}\mat{X}^\intercal \vec{1}_n$, $\mat{X}=\left[\vec{x}_1, \dots, \vec{x}_n \right]^\intercal \in \R^{n\times d}$, $\vec{u}(\vec{\theta}) = [u_1, \dots, u_n]^\intercal = \left[f(\vec{\theta}, \vec{x}_1), \dots, f(\vec{\theta}, \vec{x}_n)\right]^\intercal$ is the neural network output, and $\mat{J}(\vec{\theta})$ denotes the Jacobian matrix of $\vec{u}(\vec{\theta})$ with respect to $\vec{\theta}$.

\subsection{Convergence of \mac}
We begin by stating assumptions and lemmas required for 
establishing the convergence. Due to space constraints, 
all proofs are deferred to Appendix~\ref{apdx:converge_proof}.
As in~\cite{du2018gradient,Wu2019GlobalCO,zhang2019fast},
we make the following assumption. 
\begin{assumption}\label{as:1}
    For all $i$, $\|\vec{x}_i\| = 1$ and $|y_{i}| = \mathcal{O}(1)$. For any $i \neq j$, $\vec{x}_i \nparallel \vec{x}_j$. 
\end{assumption}
The assumption simplifies analysis and generally holds for most real-world datasets, where no two inputs are parallel. 
\begin{definition}[Limiting Gram Matrix]
    Let $\mathbf{\Sigma}^{\infty}\in \R^{n\times n}$ be defined as $\mathbf{\Sigma}^{\infty} = \mathbf{\Gamma}^\infty(\mathbf{\Gamma}^\infty)^\intercal$ with $\mathbf{\Sigma}^{\infty}_{ij} \triangleq \E_{\mathbf{w}\sim \mathcal{N}(\mathbf{0}, \mathbf{I})}[\mathbf{x}_{i}^\intercal \mathbf{x}_{j} \mathbb{I}(\mathbf{w}^\intercal \mathbf{x}_i \geq 0, \mathbf{w}^\intercal \mathbf{x}_j \geq 0)],$ where the infinite-width limit of neural tangent kernel $\mathbf{\Gamma}^\infty = \underset{m\rightarrow \infty}{\lim}\frac{1}{\sqrt{m}}[\phi'(\mathbf{X}\mathbf{w}_1), \dots, \phi'(\mathbf{X}\mathbf{w}_m)]$.
\end{definition}

The following conditions are required for the proof, and we demonstrate in the Appendix that these conditions can be satisfied by overparameterizing the network.

\begin{condition}\label{cond:1}
    $\mat{\Sigma}^{\infty}$ is positive definite with minimum eigenvalue $\lambda_{\mat{\Gamma}}$.
\end{condition}

\begin{condition}[Stable Jacobian]\label{cond:2}
    $\|\mat{J} - \mat{J}(\vec{\theta}_0)\|_2 \leq \frac{C\rho}{2\sigma_{\max}(\mat{X})}$ for some $0 < C < \frac{1}{2}$. 
\end{condition}

\begin{figure}[tb]
    \centering
    \includegraphics[width=0.7\textwidth]{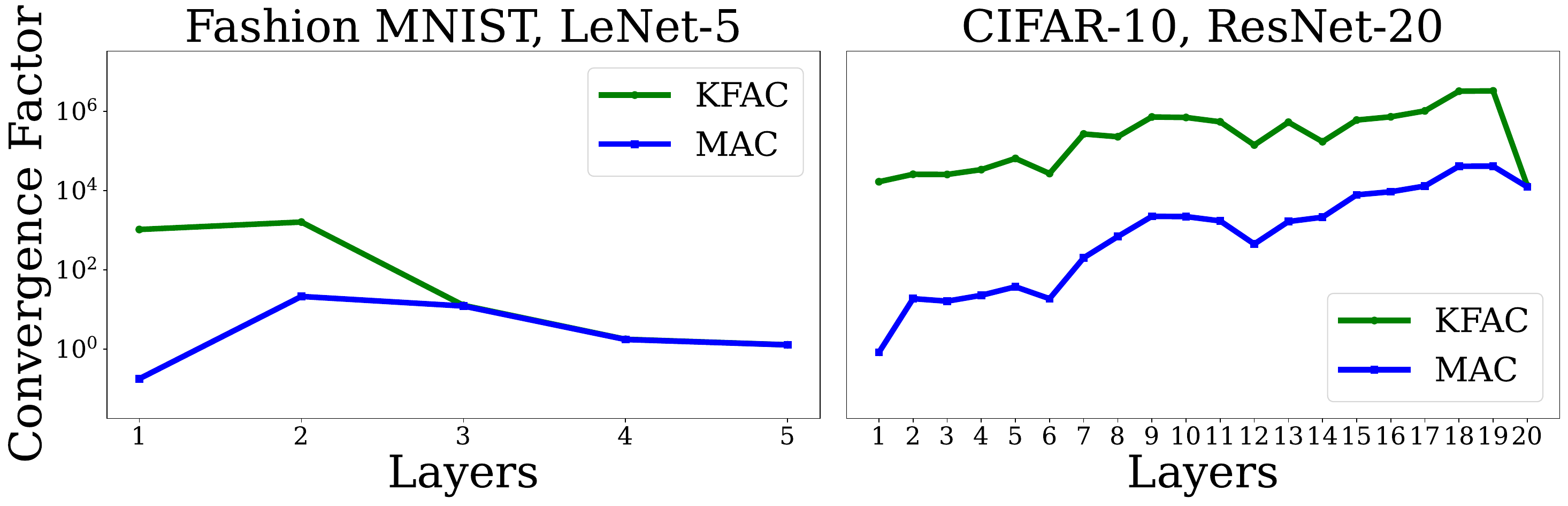}
    \caption{Comparison of the convergence factor between KFAC and \mac during (\textbf{Left}) LeNet-5 training on Fashion MNIST and (\textbf{Right}) ResNet-20 training on CIFAR-10 dataset.}
    \label{fig:convergence}
\end{figure}

We show that \mac guarantees linear convergence to a global minimum
when the neural network is sufficiently large. The main result is
stated as follows. 

\begin{theorem}[\mac]\label{thm:converge}
    Under Assumption~\ref{as:1}, Condition~\ref{cond:1}, and
    \ref{cond:2}, if we set the number of hidden units $m =
    \Omega\left(\frac{n^3\lambda_{\max}(\mat{X}^\intercal
        \mat{X})^{4}\|\vec{y}-\vec{u}(\vec{\theta}_0)\|_{2}^{2}}{\lambda_{\mat{\Gamma}}^{2}\delta^{2}\rho^3}\right)$
    and learning rate
    $\eta=\mathcal{O}\left(\frac{\rho}{\lambda_{\max}(\mat{X}^\intercal
        \mat{X})\lambda_{\mat{\Gamma}}}\right)$, then with probability
    at least $1-\delta$ over the random initialization, we have for $k
    = 0, 1, 2, \dots$ 
    \[
        \|\vec{y}-\vec{u}(\vec{\theta}_{k})\|^2 \leq \left(1 -
          \frac{\eta\lambda_{\mat{\Gamma}}\lambda_{\min}(\mat{X}^\intercal
            \mat{X})}{2(\|\overline{\vec{x}}\|^{2} + \rho)}\right)^{k}
        \|\vec{y} - \vec{u}(\vec{\theta}_{0})\|^2 
    \]
\end{theorem}

\begin{remark}
The convergence rate of \mac is captured by
$\frac{\lambda_{\max}(\bar{\mat{x}}\bar{\mat{x}}^\intercal)}{\lambda_{\min}(\mat{X}^\intercal
  \mat{X})}$, whereas the convergence rate of KFAC is characterized by
the condition number of the matrix $\mat{X}^\intercal \mat{X}$, given
by $\frac{\lambda_{\max}(\mat{X}^\intercal
  \mat{X})}{\lambda_{\min}(\mat{X}^\intercal
  \mat{X})}$~\cite{zhang2019fast}. Note that the only non-zero
eigenvalue of $\bar{\vec{x}}\bar{\vec{x}}^\intercal$ is
$\|\bar{\vec{x}}\|^2$. As shown in Figure~\ref{fig:convergence}, we
empirically find that $\|\bar{\vec{x}}\|^2$ is typically smaller than
the largest eigenvalue of $\mat{X}^\intercal \mat{X}$ but proportional
to it. This implies that \mac can achieve asymptotically the same
convergence rate with KFAC while using the computationally efficient
rank-1 approximation. 
\end{remark}
See Appendix~\ref{apdx:converge_proof} for the proof.


\section{Experiments}
\label{sec:experiment}
In this section, we evaluate the performance of 
\mac on image classification tasks using various datasets
and compare it with other baselines.
For comparison,
we focus on KFAC and its variants, excluding methods like Shampoo and
AdaHessian due to their substantial computational overhead, as
discussed in Section~\ref{sec:introduction}.
All experiments were performed using AMD EPYC 64-core CPUs and 4 Nvidia A100 GPUs.

\subsection{Training Settings}
For CIFAR dataset, we utilized ResNet-110,
DenseNet-121~\cite{Huang2016DenselyCC}, and
WideResNet-28-10~\cite{Zagoruyko2016WideRN}. Each 
model was trained for 
50/100/200 epochs with a mini-batch size of 128 and 
cosine annealing~\cite{Loshchilov2016SGDRSG} 
learning rate scheduling 
on a single GPU.
The reported metrics are averages from 5 different runs.
We compare our methods against SGD momentum, AdamW~\cite{kingma2015adam,Loshchilov2019DecoupledWD}, KFAC, and the
state-of-the-art KFAC variants, FOOF, Eva, and LNGD. 
For KFAC and its variants, including \mac,
we set the frequency of approximated FIM inversion to
50 iterations. 
For ImageNet-1k~\cite{deng2009imagenet} experiments, 
we trained
ResNet-50, ResNet-101, DeiT-Small~\cite{Touvron2020TrainingDI}, and Swin Transformer Tiny
(Swin-Tiny)~\cite{Liu2021SwinTH} for 100/200 epochs with a mini-batch
size of 1,024 and cosine learning rate scheduling using 4
GPUs. 
See Table~\ref{tab:imagenet_setting} and Table~\ref{tab:imagenet_hyperpar} in Appendix~\ref{apdx:imagenet_detail} for the experimental details.

\begin{table*}[tb] 
\caption{Test accuracy (\%) and standard deviation (in parentheses)  on CIFAR-10 (\textbf{Top}) and CIFAR-100
    (\textbf{Bottom}) datasets across different optimizers and models: ResNet-110,
    DenseNet-121, and WideResNet-28-10.}
\label{tab:cifar}
\centering
\begin{threeparttable}
\footnotesize
\begin{tabular}{c|ccc|ccc|ccc}
\toprule
Model & \multicolumn{3}{c|}{ResNet-110} & \multicolumn{3}{c|}{DenseNet-121} & \multicolumn{3}{c}{WideResNet-28-10} \\
Epoch & 50 & 100 & 200 & 50 & 100 & 200 & 50 & 100 & 200 \\
\midrule
\mac    & 93.5\scriptsize{(0.2)} & 94.4\scriptsize{(0.2)} & 94.9\scriptsize{(0.2)} & 95.3\scriptsize{(0.2)} & 95.7\scriptsize{(0.1)} & 95.9\scriptsize{(0.1)} & 95.7\scriptsize{(0.1)} & 96.2\scriptsize{(0.1)} & 96.4\scriptsize{(0.1)}  \\
\midrule
\textsc{SGD}    & 92.0\scriptsize{(0.3)} & 93.5\scriptsize{(0.4)} & 94.3\scriptsize{(0.3)} & 94.9\scriptsize{(0.2)} & 95.4\scriptsize{(0.1)} & 95.6\scriptsize{(0.1)} & 95.4\scriptsize{(0.1)} & 96.0\scriptsize{(0.2)} & 96.2\scriptsize{(0.0)}  \\
\textsc{AdamW}  & 93.4\scriptsize{(0.1)} & 94.2\scriptsize{(0.1)} & 94.4\scriptsize{(0.1)} & 94.7\scriptsize{(0.1)} & 94.9\scriptsize{(0.2)} & 94.9\scriptsize{(0.2)} & 95.1\scriptsize{(0.1)} & 95.6\scriptsize{(0.1)} & 95.9\scriptsize{(0.1)}  \\
\midrule
\textsc{KFAC}   & 93.5\scriptsize{(0.1)} & 94.3\scriptsize{(0.1)} & 94.7\scriptsize{(0.2)} & 94.9\scriptsize{(0.1)} & 95.3\scriptsize{(0.1)} & 95.6\scriptsize{(0.1)} & 95.4\scriptsize{(0.1)} & 96.0\scriptsize{(0.1)} & 96.3\scriptsize{(0.1)}  \\
\textsc{FOOF}   & 94.0\scriptsize{(0.1)} & 94.7\scriptsize{(0.1)} & 95.1\scriptsize{(0.1)} & 95.6\scriptsize{(0.1)} & 95.8\scriptsize{(0.1)} & 96.0\scriptsize{(0.1)} & 95.8\scriptsize{(0.1)} & 96.2\scriptsize{(0.1)} & 96.4\scriptsize{(0.0)}  \\
\textsc{Eva}    & 93.6\scriptsize{(0.2)} & 94.1\scriptsize{(0.1)} & 94.7\scriptsize{(0.1)} & 94.7\scriptsize{(0.1)} & 95.3\scriptsize{(0.1)} & 95.7\scriptsize{(0.2)} & 95.4\scriptsize{(0.2)} & 95.9\scriptsize{(0.1)} & 96.2\scriptsize{(0.2)}  \\
\textsc{LNGD}   & 92.8\scriptsize{(0.1)} & 93.8\scriptsize{(0.2)} & 94.1\scriptsize{(0.1)} & 95.0\scriptsize{(0.1)} & 95.4\scriptsize{(0.2)} & 95.3\scriptsize{(0.2)} & 95.2\scriptsize{(0.2)} & 95.6\scriptsize{(0.1)} & 95.8\scriptsize{(0.2)} \\
\bottomrule
\toprule
Model & \multicolumn{3}{c|}{ResNet-110} & \multicolumn{3}{c|}{DenseNet-121} & \multicolumn{3}{c}{WideResNet-28-10}  \\
Epoch & 50 & 100 & 200 & 50 & 100 & 200 & 50 & 100 & 200 \\
\midrule
\mac    & 72.8\scriptsize{(0.4)} & 74.2\scriptsize{(0.4)} & 75.0\scriptsize{(0.3)} & 78.9\scriptsize{(0.2)} & 80.5\scriptsize{(0.2)} & 80.6\scriptsize{(0.2)} & 79.4\scriptsize{(0.2)} & 80.9\scriptsize{(0.2)} & 81.6\scriptsize{(0.2)}  \\
\midrule
\textsc{SGD}    & 71.1\scriptsize{(1.0)} & 72.5\scriptsize{(0.7)} & 73.5\scriptsize{(0.6)} & 78.2\scriptsize{(0.2)} & 79.6\scriptsize{(0.1)} & 79.9\scriptsize{(0.3)} & 79.3\scriptsize{(0.2)} & 80.7\scriptsize{(0.1)} & 81.5\scriptsize{(0.2)}  \\
\textsc{AdamW}  & 71.6\scriptsize{(0.1)} & 73.4\scriptsize{(0.2)} & 73.7\scriptsize{(0.3)} & 77.3\scriptsize{(0.2)} & 78.5\scriptsize{(0.3)} & 79.0\scriptsize{(0.2)} & 78.2\scriptsize{(0.2)} & 79.7\scriptsize{(0.1)} & 80.2\scriptsize{(0.2)}  \\
\midrule
\textsc{KFAC}   & 71.5\scriptsize{(1.3)} & 73.2\scriptsize{(0.5)} & 74.2\scriptsize{(0.4)} & 78.1\scriptsize{(0.5)} & 79.7\scriptsize{(0.2)} & 80.1\scriptsize{(0.3)} & 79.3\scriptsize{(0.2)} & 81.0\scriptsize{(0.2)} & 81.5\scriptsize{(0.1)}  \\
\textsc{FOOF}   & 73.6\scriptsize{(0.2)} & 75.1\scriptsize{(0.3)} & 76.0\scriptsize{(0.4)} & 79.8\scriptsize{(0.1)} & 80.9\scriptsize{(0.2)} & 81.0\scriptsize{(0.2)} & 80.0\scriptsize{(0.2)} & 80.8\scriptsize{(0.2)} & 81.0\scriptsize{(0.3)}  \\
\textsc{Eva}    & 71.9\scriptsize{(0.6)} & 73.6\scriptsize{(0.3)} & 74.7\scriptsize{(0.3)} & 77.9\scriptsize{(0.3)} & 79.4\scriptsize{(0.3)} & 79.9\scriptsize{(0.2)} & 79.3\scriptsize{(0.3)} & 81.0\scriptsize{(0.2)} & 81.6\scriptsize{(0.3)}  \\
\textsc{LNGD}   & 71.7\scriptsize{(0.2)} & 73.1\scriptsize{(0.3)} & 74.5\scriptsize{(0.1)} & 78.8\scriptsize{(0.1)} & 79.9\scriptsize{(0.1)} & 79.7\scriptsize{(0.3)} & 79.1\scriptsize{(0.2)} & 79.8\scriptsize{(0.2)} & 79.4\scriptsize{(0.2)}  \\
\bottomrule
\end{tabular}
\begin{tablenotes}
\footnotesize
\item *Ranks are calculated on the average test accuracy across all epochs.
\end{tablenotes}
\end{threeparttable}
\end{table*}

\begin{figure*}[tb]
    \centering
    \includegraphics[width=0.98\textwidth]{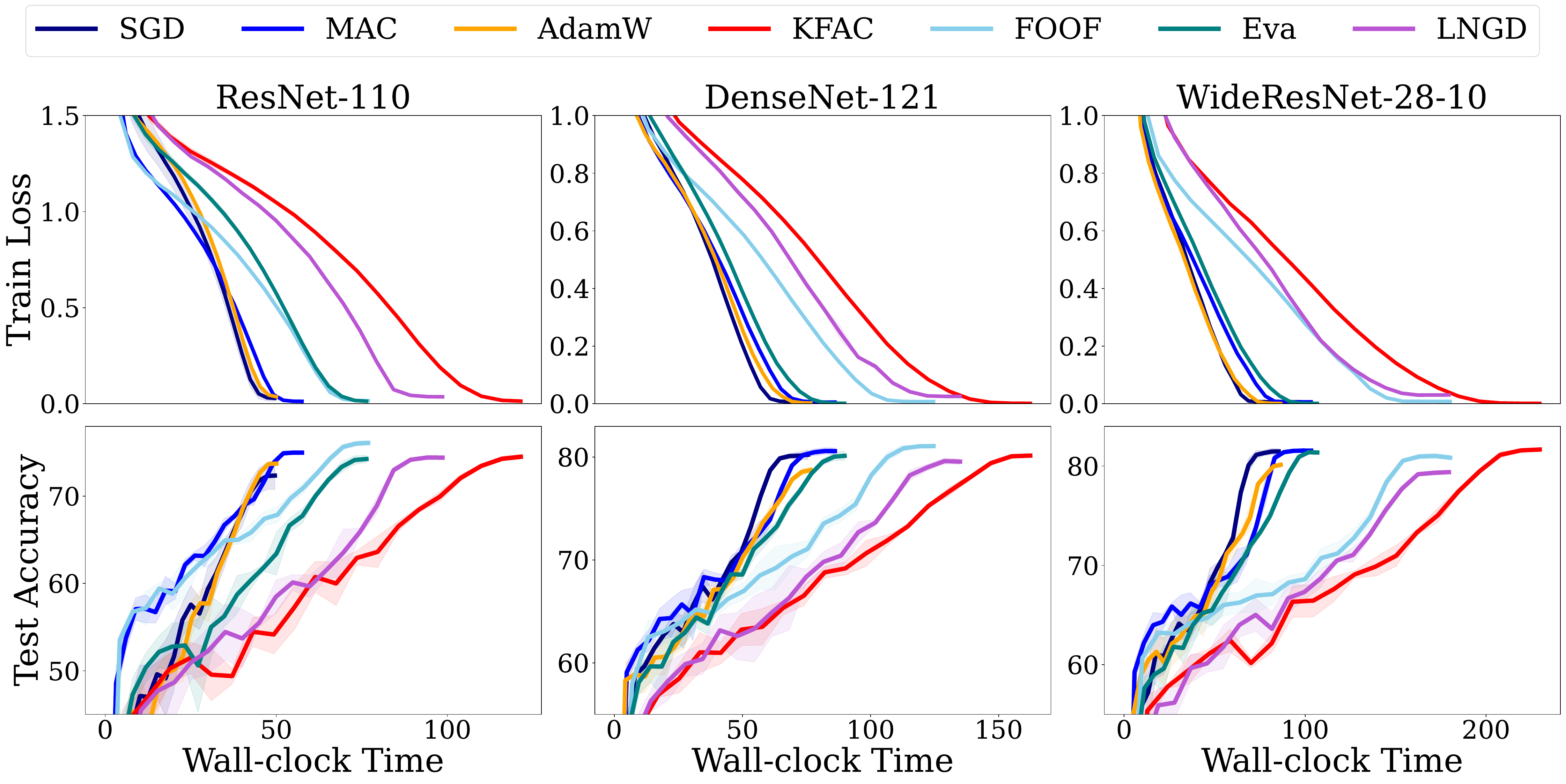}
    \caption{Comparison of train loss and test accuracy over wall-clock time on CIFAR-100 dataset.}
    \label{fig:cifar100}
\end{figure*}

\begin{table*}[tb]
\caption{Top-1 accuracy (\%) of ResNets and ViTs on ImageNet-1k.}
\label{tab:imgnet_resnet}
\centering
\begin{threeparttable}
\small
\begin{tabular}{c|cc|cc|cc|cc|cc}
\toprule
Model & \multicolumn{2}{c|}{ResNet-50} & \multicolumn{2}{c|}{ResNet-101} & \multicolumn{2}{c|}{DeiT-Small} & \multicolumn{2}{c|}{Swin-Tiny} & Avg. & Rank \\
Epoch & 100 & 200 & 100 & 200 & 100 & 200 & 100 & 200 &  & \\
\midrule
\mac      & 78.0 & 79.7 & 79.8 & 81.2 & 73.5 & 77.4 & 77.7 & 80.1 & \textbf{78.4} & \textbf{1} \\
\midrule
\textsc{SGD}    & 78.1 & 79.6 & 79.7 & 81.3 & 69.1 & 75.3 & 76.0 & 78.9 & 77.3 & 3 \\
\textsc{AdamW}  & 76.8 & 79.2 & 77.9 & 80.6 & 73.7 & 77.9 & 77.4 & 80.1 & 78.0 & 2 \\
\midrule
\textsc{KFAC}   & 78.2 & 79.3 & 79.6 & 81.1 & 69.9 & \xmark & \xmark & \xmark & 48.5 & 6 \\
\textsc{FOOF}   & 78.4 & 79.7 & 80.0 & 81.0 & 63.9 & 67.9 & 73.7 & 72.6 & 74.7 & 4 \\
\textsc{Eva}    & 77.7 & 79.4 & 79.6 & 81.1 & 69.7 & 76.6 & \xmark & \xmark & 58.0 & 5 \\
\bottomrule
\end{tabular}
\begin{tablenotes}
\footnotesize
\item *Ranks are calculated on the average top-1 accuracy.
\item **\xmark \ indicates a training failure.
\end{tablenotes}
\end{threeparttable}
\end{table*}

\subsection{Results on CIFAR}
As summarized in Tables~\ref{tab:cifar}, \mac demonstrates consistent
improvements in test accuracy and convergence speed across all
evaluated networks, achieving a second-place ranking on
average. Notably, on the CIFAR-100 dataset, \mac attains significantly
higher average test accuracy compared to both first-order and other
second-order methods, closely matching the performance of
FOOF. Although the accuracy gap between \mac and FOOF is minor, our
implementation of \mac is 20.6\% to 43.2\% faster in training time,
highlighting a favorable trade-off between efficiency and
effectiveness. 
Figure~\ref{fig:cifar100} describes the change of training loss and test accuracy over wall-clock time.
%
Remarkably,
\mac’s execution time is comparable to that of SGD and AdamW,
yet it achieves higher test accuracy than these first-order
methods.
%
In contrast, although Eva runs within a similar time frame as \mac, it
consistently yields lower accuracy.
Moreover, FOOF, LNGD, and
KFAC require significantly longer execution times, highlighting the
efficiency advantage of \mac. 

\begin{table*}[tb]
\caption{Comparison of relative wall-clock time and memory usage over SGD (1.00) on CIFAR and ImageNet-1k.}
\label{tab:time_mem}
\centering
\small
\begin{tabular}{l|rr|rr|rr|rr|rr}
\toprule
Model (\# params) & \multicolumn{2}{c|}{\mac}  & \multicolumn{2}{c|}{\textsc{KFAC}} & \multicolumn{2}{c|}{\textsc{FOOF}} & \multicolumn{2}{c|}{\textsc{Eva}} & \multicolumn{2}{c}{\textsc{LNGD}} \\
 & Time & Mem & Time & Mem & Time & Mem & Time & Mem & Time & Mem\\
\midrule
ResNet-110 (2M)  & \textbf{1.16} & 1.03 & 2.45 & 1.10 & 1.55 & 1.08 & 1.54 & 1.03 & 1.90 & 4.98\\
DenseNet-121 (8M)  & \textbf{1.14} & 1.00 & 2.18 & 1.04 & 1.64 & 1.04 & 1.19 & 1.00 & 1.79 &  2.26\\
WRN-28-10 (37M)  & \textbf{1.21} & 1.05 & 2.71 & 1.31 & 2.13 & 1.30 & 1.25 & 1.06 & 2.10 &  12.48\\
\midrule
ResNet-50 (27M)  & \textbf{1.04} & 1.00 & 1.25 & 1.02 & 1.21 & 1.01 & 1.06 & 1.00 & - & - \\
ResNet-101 (45M)  & \textbf{1.06} & 1.00 & 1.37 & 1.02 & 1.31 & 1.02 & 1.09 & 1.00 & - & - \\
DeiT-Small (22M)  & \textbf{1.03} & 1.00 & 1.26 & 1.04 & 1.25 & 1.04 & 1.14 & 1.00 & - & - \\
Swin-Tiny (28M)  & \textbf{1.02} & 1.00 & 1.12 & 1.03 & 1.13 & 1.02 & 1.06 & 1.00 & - & - \\
\bottomrule
\end{tabular}
\end{table*}

\subsection{Results on ImageNet}
The experimental results on ImageNet-1k across ResNets and ViTs, summarized in
Table~\ref{tab:imgnet_resnet}, demonstrate that \mac
performs competitively or surpass other baseline methods across various
models and training durations. Overall, \mac achieves the highest
average accuracy (78.4\%) and is ranked first. While all optimizers
show similar performance on ResNet architectures, the advantages of
\mac become more evident on vision transformers. In particular, on
DeiT-Small and Swin-Tiny, \mac consistently 
yields higher accuracy
than SGD and closely matches the performance of AdamW. In contrast,
KFAC, FOOF, and Eva exhibit unstable training on vision transformers,
leading to significantly lower average accuracy. FOOF, for instance,
suffers from unstable inversion of the preconditioner: using a small
damping term results in training failures, whereas a larger damping
term degrades performance. These results highlight \mac’s ability to
combine high accuracy with robust and efficient convergence,
especially in transformer architectures, where conventional
second-order methods often struggle.
LNGD was excluded from the ImageNet experiments due to its prohibitive
memory usage.

\subsection{Time and Memory Complexities}
Table~\ref{tab:time_mem} highlights the computational efficiency gains
of \mac. As shown, \mac achieves the \emph{fastest}
execution time among all KFAC variants, outperforming the
state-of-the-art lightweight KFAC variant, Eva. Specifically, \mac
reduces training time of ResNet-110 by up to 52.7\% compared to KFAC,
25.2\% compared to FOOF, 24.7\% compared to Eva, and 38.9\% compared
to LNGD, while maintaining nearly the same memory usage as SGD.
Particularly, on ImageNet-1k, \mac incurs only a 2-6\% increase in
time complexity, while using almost the same memory as SGD.

\subsection{Ablation Study on Hyperparameters}
\label{apdx:ablation}

\begin{figure*}[tb]
    \begin{subfigure}[t]{0.24\linewidth}
        \centering
        \includegraphics[width=\linewidth]{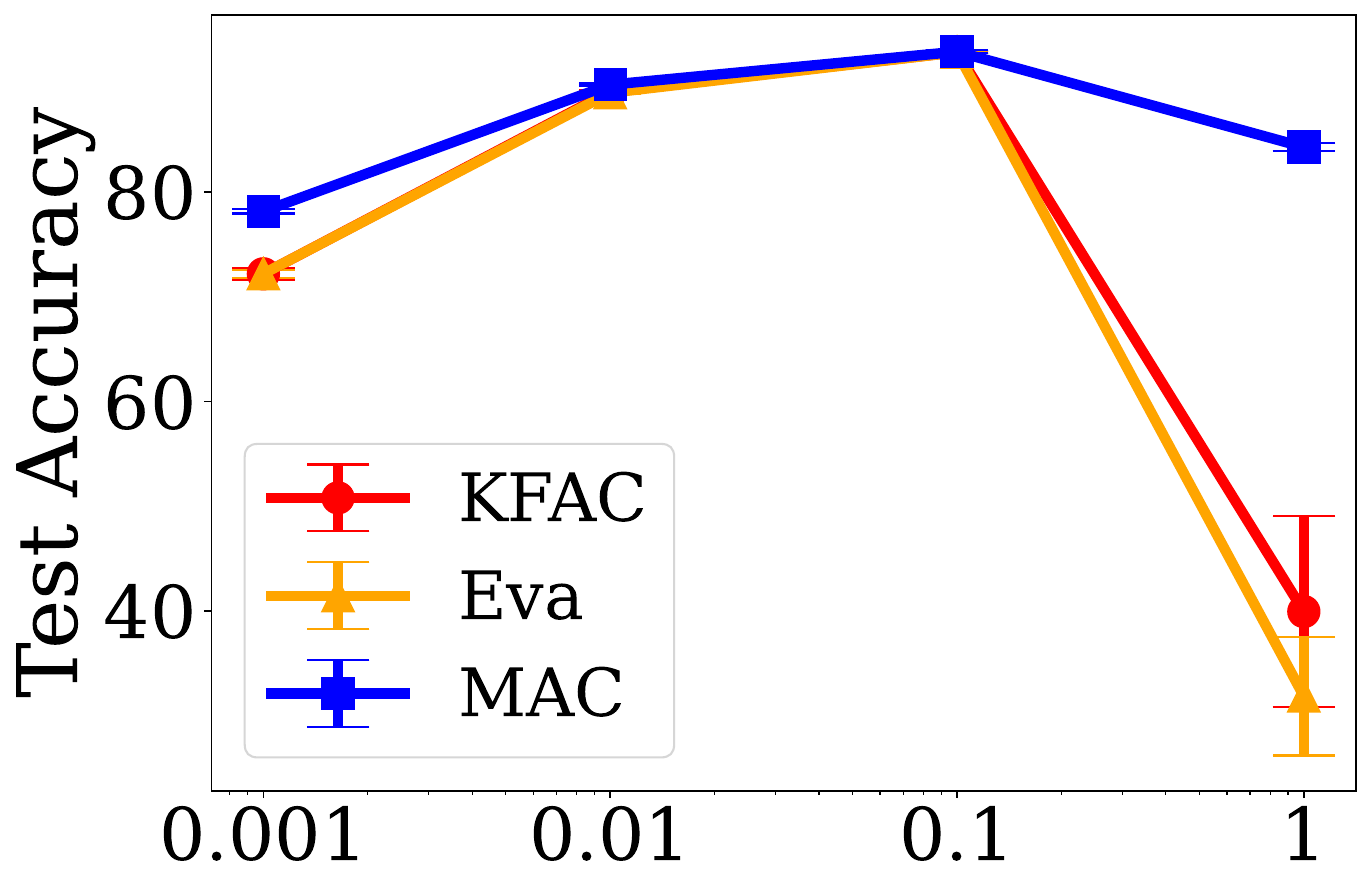}
        \caption{Learning rate}
        \label{fig:hyperparam_lr}
    \end{subfigure}
    \hfill
    \begin{subfigure}[t]{0.24\linewidth}
        \centering
        \includegraphics[width=\linewidth]{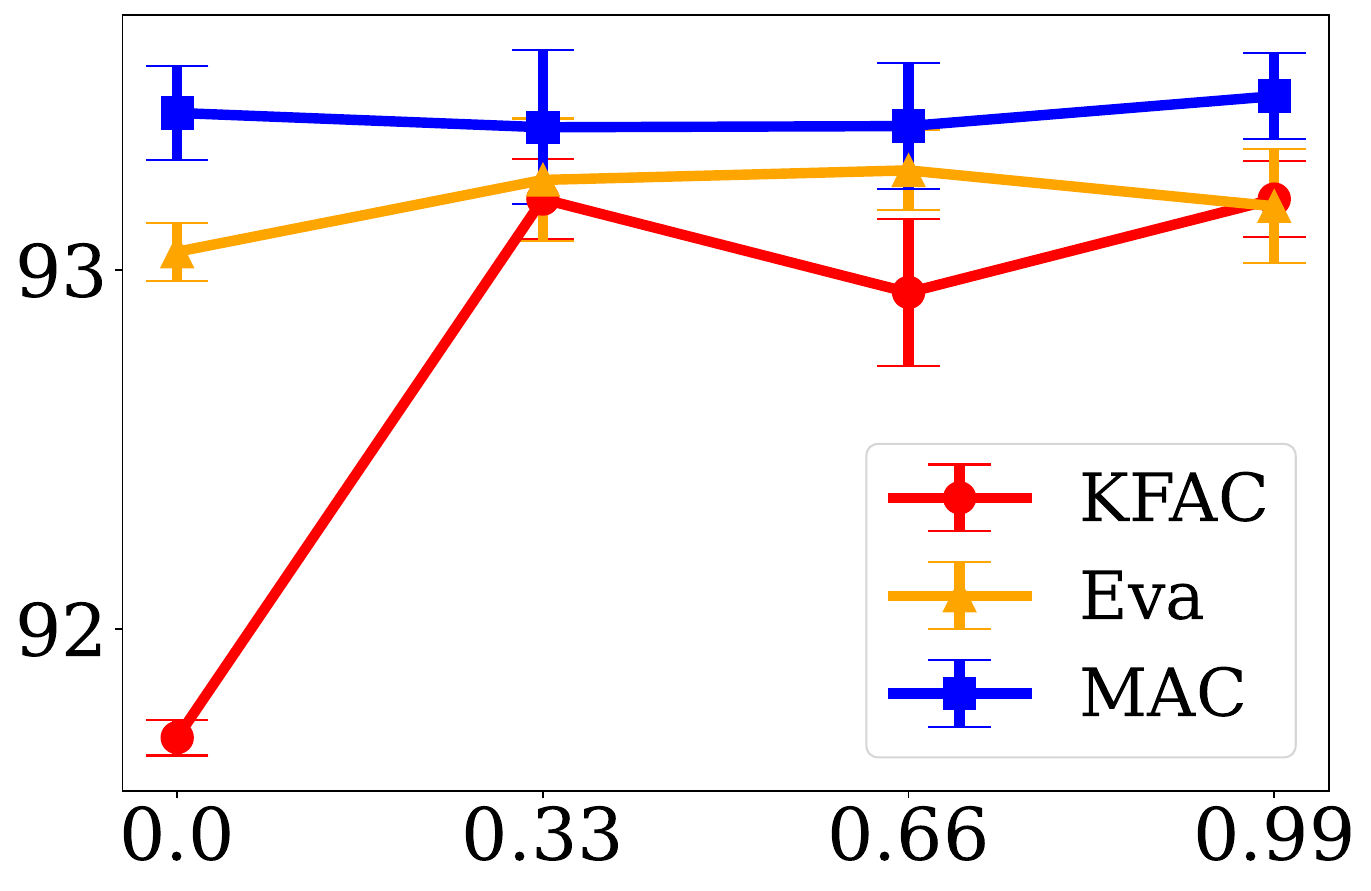}
        \caption{EMA coefficient}
        \label{fig:hyperparam_ema}
    \end{subfigure}
    \hfill
    \begin{subfigure}[t]{0.221\linewidth}
        \centering
        \includegraphics[width=\linewidth]{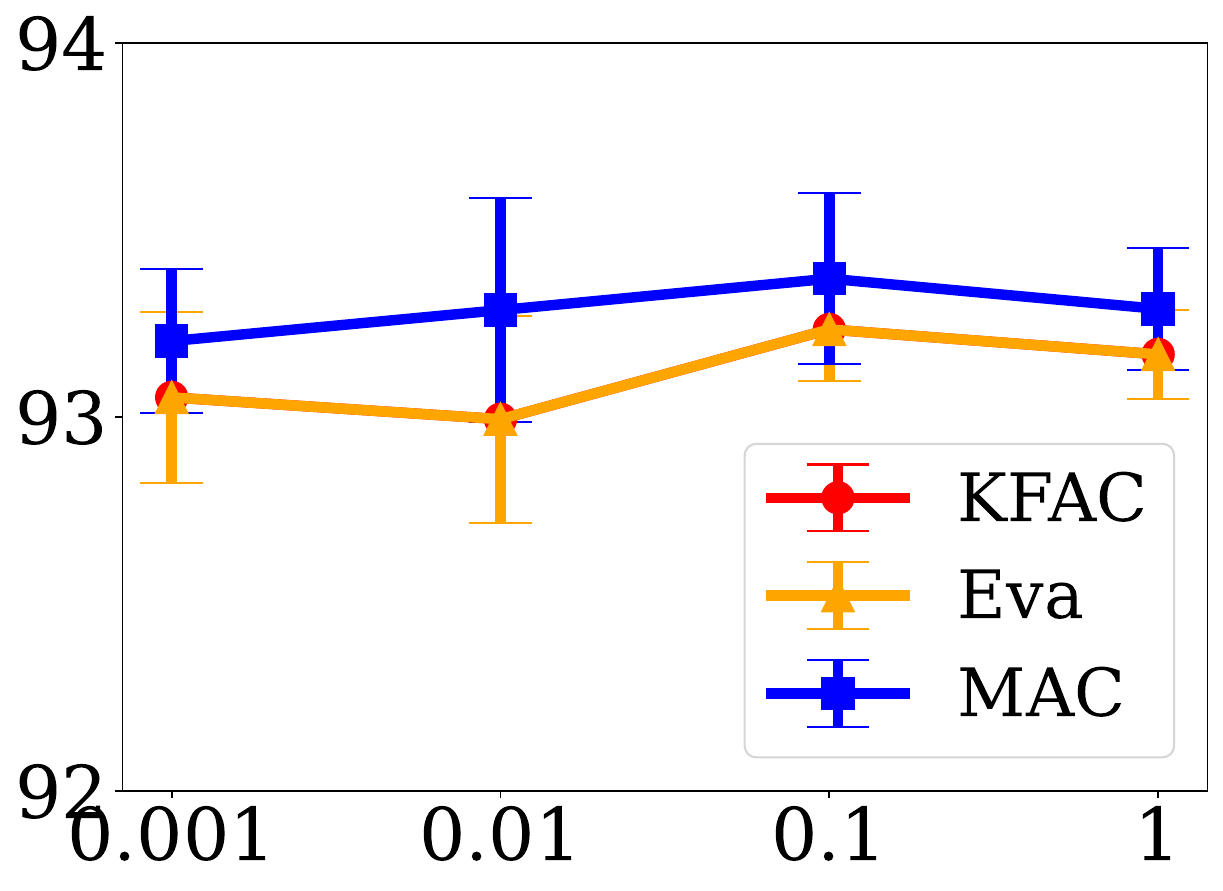}
        \caption{Damping}
        \label{fig:hyperparam_damp}
    \end{subfigure}
    \hfill
    \begin{subfigure}[t]{0.24\linewidth}
        \centering
        \includegraphics[width=\linewidth]{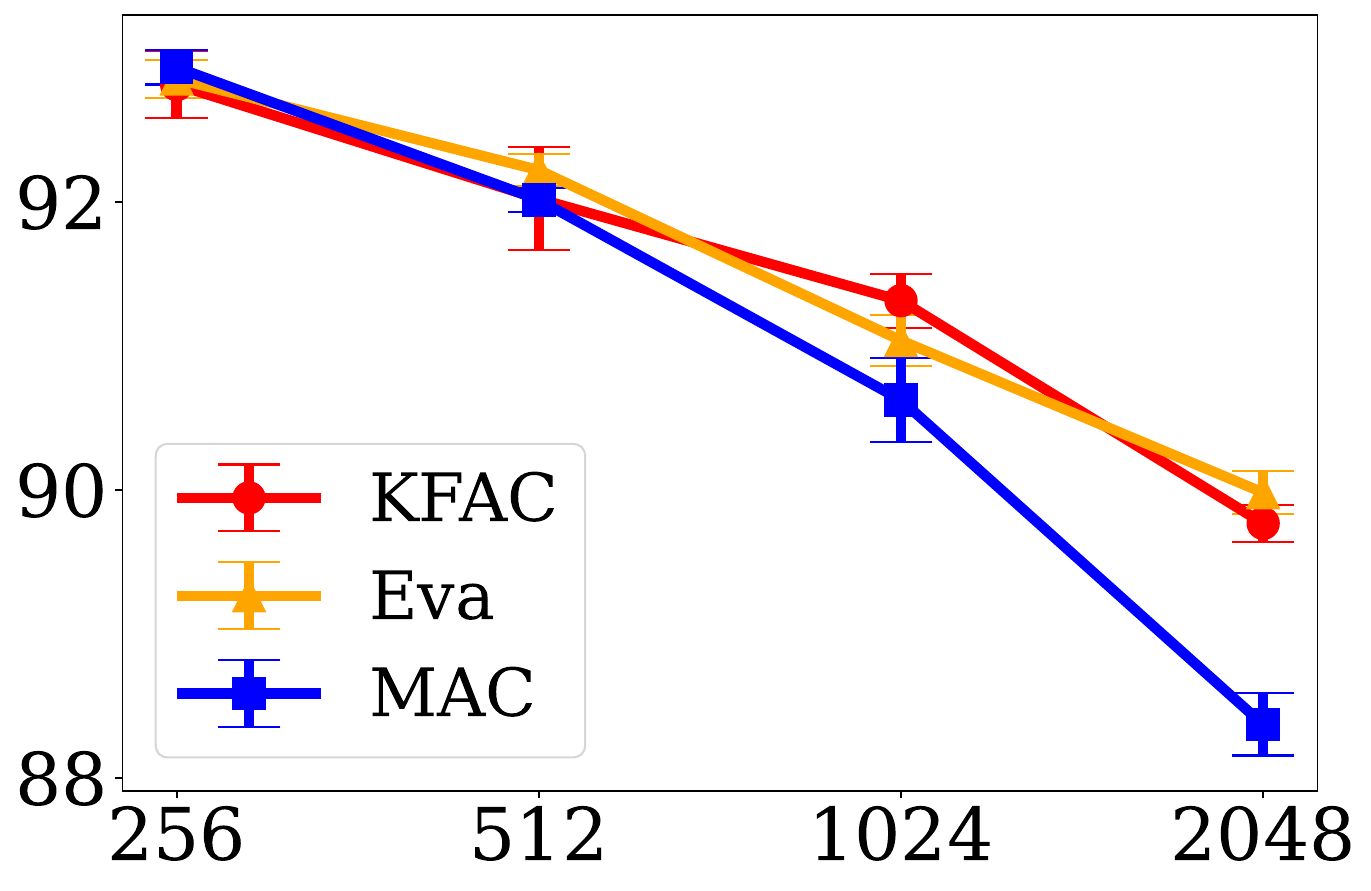}
        \caption{Batch size}
        \label{fig:hyperparam_batch}
    \end{subfigure}
    \caption{Hyperparameter sensitivity analysis of \mac compared to KFAC and Eva on CIFAR-10 using ResNet-32, trained for 100 epochs. Subplots show test accuracy across variations in (a) learning rate, (b) EMA coefficient, (c) damping, and (d) batch size.}
    \label{fig:ablation}
\end{figure*}

\begin{figure*}[tb]
    \begin{subfigure}[t]{0.2725\linewidth}
        \centering
        \includegraphics[width=\linewidth]{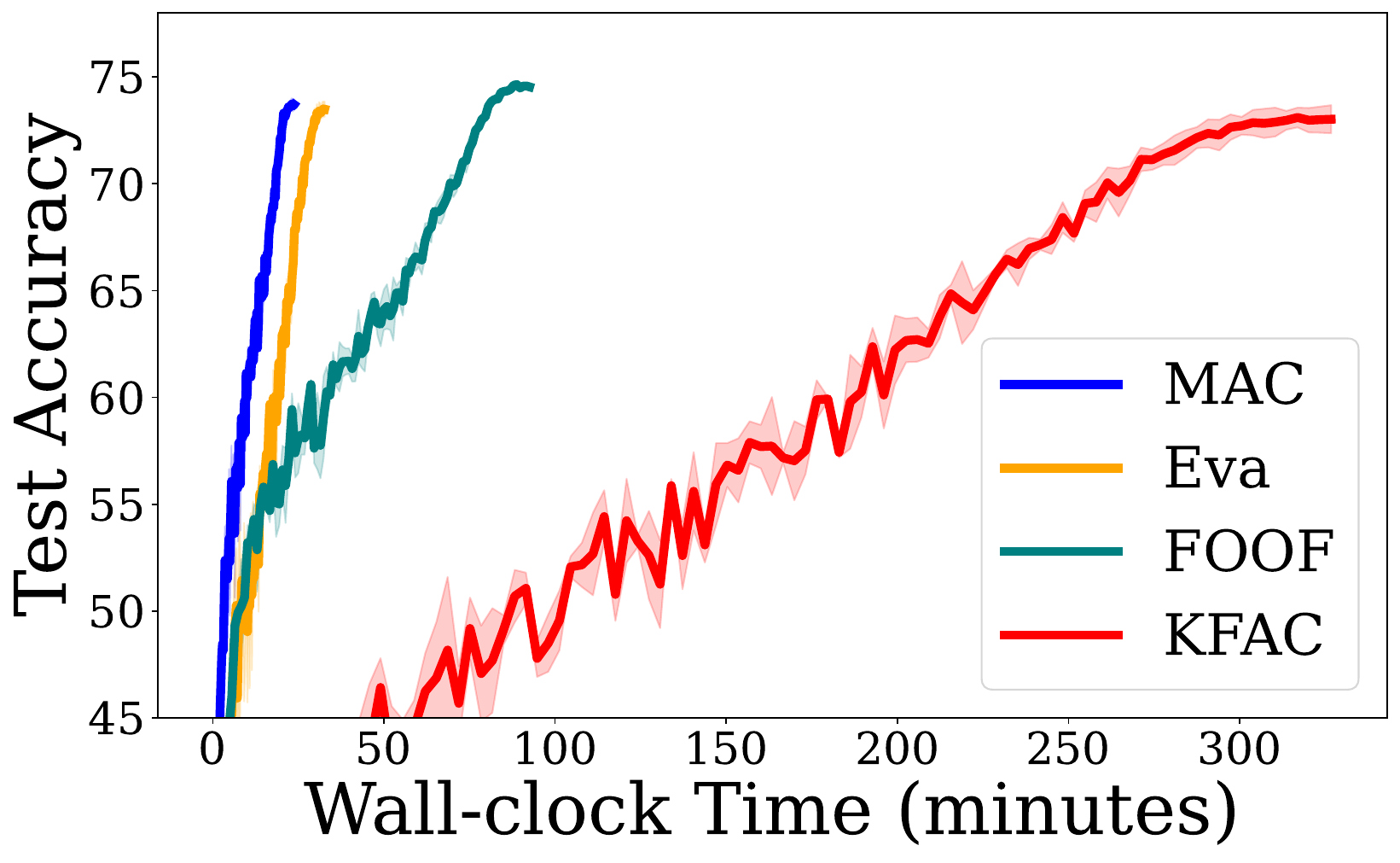}
        \caption{$\tau = 1$}
        \label{fig:inv_freq_1}
    \end{subfigure}
    \hfill
    \begin{subfigure}[t]{0.235\linewidth}
        \centering
        \includegraphics[width=\linewidth]{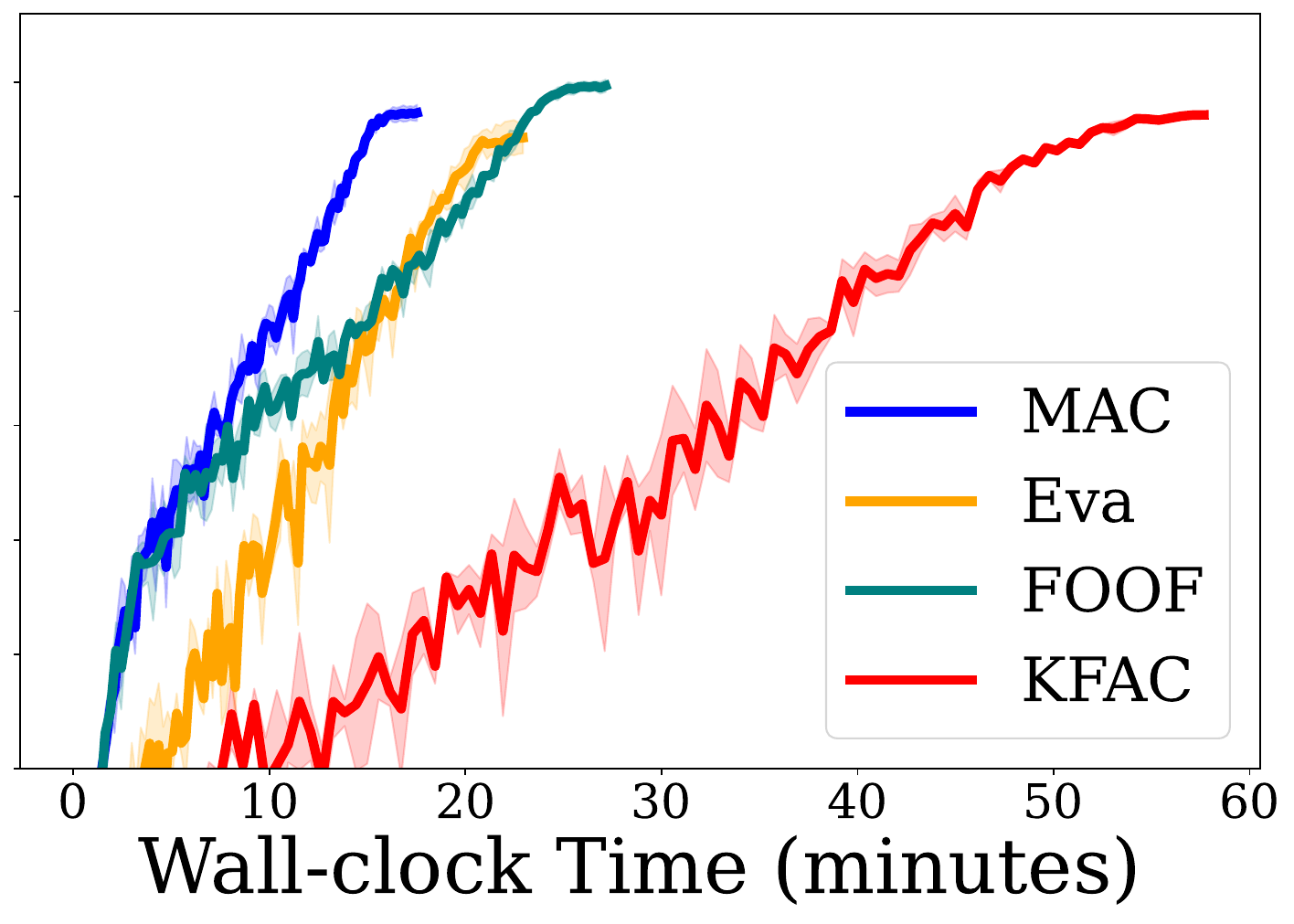}
        \caption{$\tau = 10$}
        \label{fig:inv_freq_10}
    \end{subfigure}
    \hfill
    \begin{subfigure}[t]{0.235\linewidth}
        \centering
        \includegraphics[width=\linewidth]{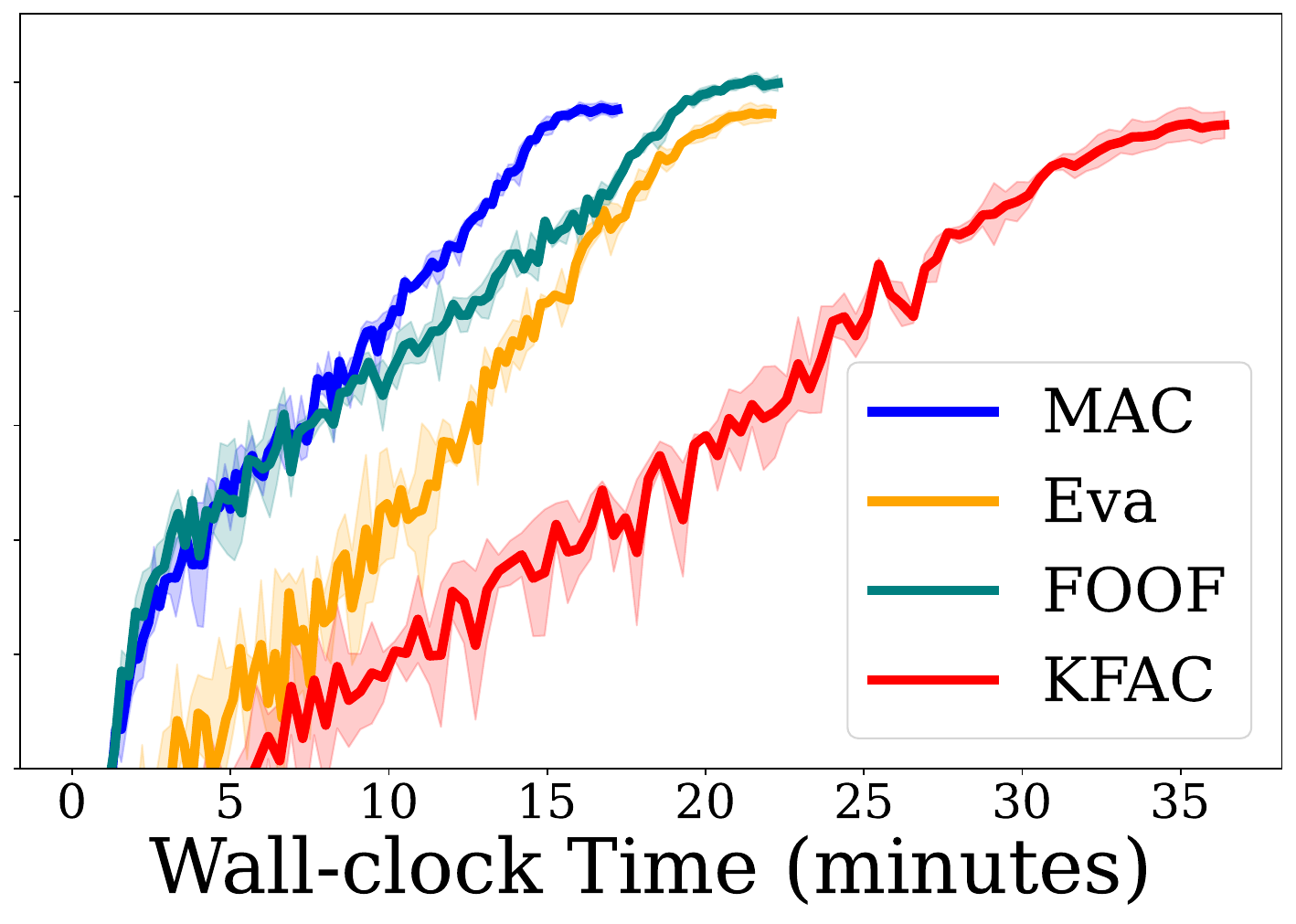}
        \caption{$\tau = 50$}
        \label{fig:inv_freq_50}
    \end{subfigure}
    \hfill
    \begin{subfigure}[t]{0.235\linewidth}
        \centering
        \includegraphics[width=\linewidth]{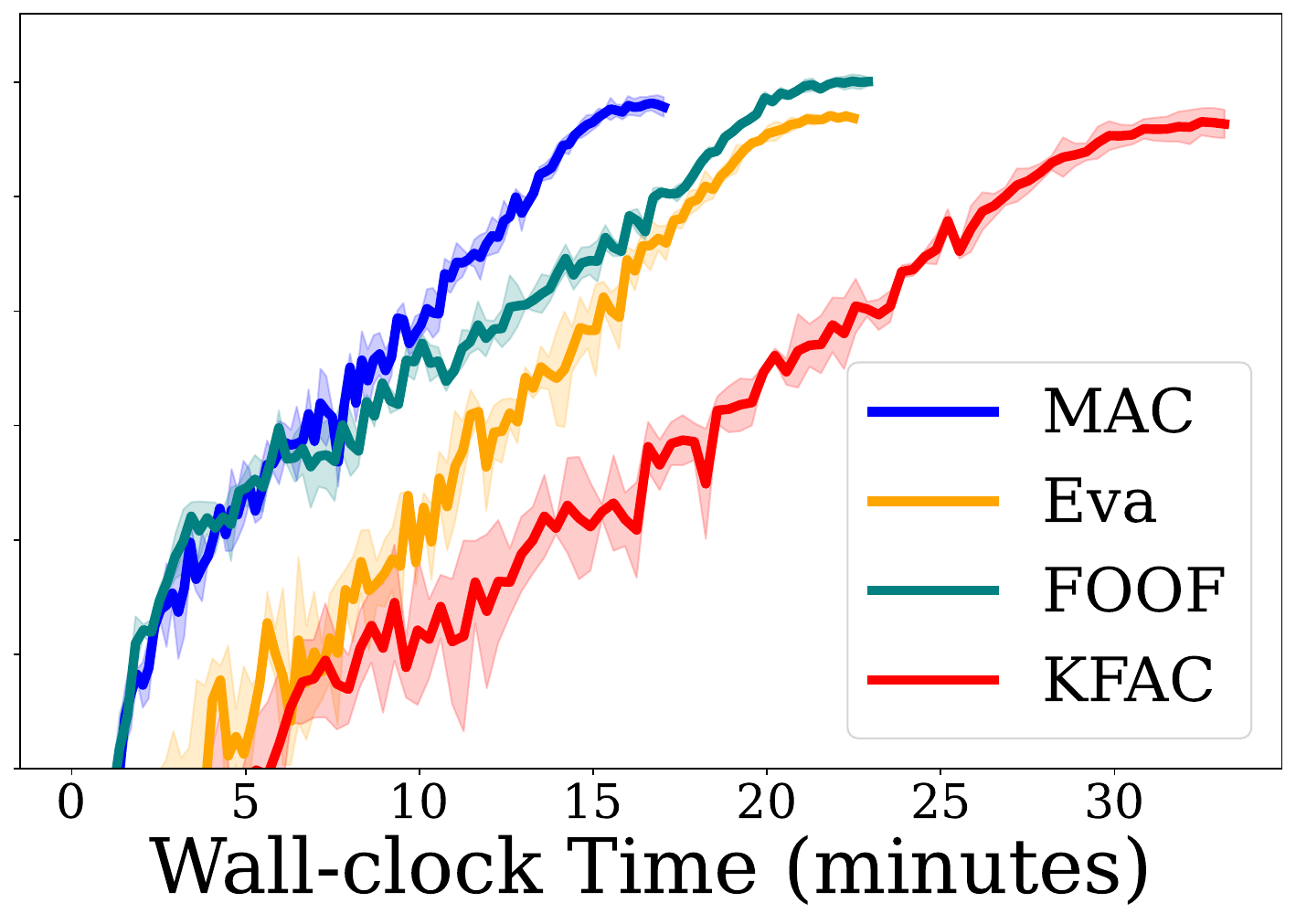}
        \caption{$\tau = 100$}
        \label{fig:inv_freq_100}
    \end{subfigure}
    \caption{Wall-clock time comparison of different inverse update frequencies (1, 10, 50, and 100 steps) during ResNet-110 training on the CIFAR-100 dataset.}
    \label{fig:cifar100_inv_freq}
\end{figure*}

Figure~\ref{fig:ablation} provides a comprehensive hyperparameter sensitivity analysis of \mac, compared against KFAC and Eva using ResNet-32 on the CIFAR-10 dataset over 100 training epochs, averaged across 3 independent runs. The four subplots illustrate how test accuracy is affected by variations in (a) learning rate, (b) EMA coefficient, (c) damping, and (d) mini-batch size.

In subplot (a), \mac achieves peak performance at a learning rate of 0.1 and maintains competitive accuracy across the entire range of values. In contrast, both KFAC and Eva exhibit substantial drops in performance at a learning rate of 1.0, highlighting \mac's robustness. Subplot (b) demonstrates that \mac is largely insensitive to the choice of EMA coefficient, maintaining stable accuracy. This contrasts with KFAC, which shows notable fluctuations, indicating higher sensitivity. Subplot (c) confirms \mac's stability across different damping values, with performance remaining consistent. Lastly, subplot (d) shows that while all methods experience reduced accuracy at larger batch sizes, \mac exhibits a slightly more pronounced decline. Nevertheless, in large-scale settings such as ImageNet training with a batch size of 1,024, \mac still matches or exceeds the performance of KFAC and Eva, despite this observed trend.

As shown in Figure~\ref{fig:cifar100_inv_freq}, we evaluated inverse update frequencies of 1, 10, 50, and 100 steps on the CIFAR-100 dataset using ResNet-110, trained for 100 epochs across 3 independent runs. We found that increasing the update frequency had minimal impact on the test accuracy of \mac, while substantially reducing computational overhead. Among all methods, KFAC consistently showed the slowest training time across all frequencies, followed by FOOF. At update frequencies of 50 and 100, Eva's training speed closely matched that of FOOF, both significantly slower than \mac. Notably, \mac consistently achieved the fastest execution time across all evaluated frequencies.

\section{Conclusions}
\label{sec:conclusion}
In this paper, 
we introduced \mac, an efficient 
%
preconditioned gradient method.
%
Our empirical analysis on eigenspectra of KFs sheds light on the
structural properties of curvature 
matrices used in KFAC,
providing insight for the development of a new optimization
method. Extensive empirical evaluations on the image 
classification 
tasks support the claim that our method enjoys the high
performance of second-order methods while being as scalable as
first-order methods. Moreover, we derived an accurate formulation of
KFAC’s FIM for self-attention layers in transformers
and effectively integrate attention scores into the precoditioning.
%


\bibliographystyle{plain}
\bibliography{main}

\newpage
\appendix
\section*{Appendix}
\addcontentsline{toc}{section}{Appendix}

\section{Proof of Proposition~\ref{prop:1}}
\begin{proof}
    Let $\mat{X}$ be an $m \times n$ matrix and $\bar{\vec{x}}\in \R^n$ be the mean vector of $\mat{X}$. Define a perturbation matrix $\mat{E}$ such that
    \begin{equation*}
        \mat{X} = \vec{1}_m \bar{\vec{x}}^\intercal + \mat{E},
    \end{equation*}
    where $\mat{E}$ has rows given by deviations from the mean. We have
    \begin{align*}
        \mat{X}^\intercal \mat{X} &= (\vec{1}_m\bar{\vec{x}}^\intercal + \mat{E})^\intercal(\vec{1}_m\bar{\vec{x}}^\intercal + \mat{E}) \\
        &= \bar{\vec{x}}\vec{1}_m^\intercal \vec{1}_m \bar{\vec{x}}^\intercal + \bar{\vec{x}}\vec{1}_m^\intercal\mat{E} + \mat{E}^\intercal \vec{1}_m \bar{\vec{x}}^\intercal + \mat{E}^\intercal \mat{E} \\
        &= m\bar{\vec{x}}\bar{\vec{x}}^\intercal + \underbracket{\bar{\vec{x}}\vec{1}_m^\intercal\mat{E} + \mat{E}^\intercal \vec{1}_m \bar{\vec{x}}^\intercal + \mat{E}^\intercal \mat{E}}_{\circled{A}}.
    \end{align*}
    The term $m\bar{\vec{x}}\bar{\vec{x}}^\intercal$ has a rank-1 structure, with its top eigenvalue being $m\|\bar{\vec{x}}\|^2$, and the corresponding eigenvector is $\bar{\vec{x}}$. To approximate $\mat{X}^\intercal \mat{X}$ as $m\bar{\vec{x}}\bar{\vec{x}}^\intercal$, the contributions of the term $\circled{A}$ should be small compared to $m\bar{\vec{x}}\bar{\vec{x}}^\intercal$. Defining $\mat{B} = \bar{\vec{x}}\vec{1}_m^\intercal \mat{E}$, we have
    \begin{align*}
        \|\mat{B}\|_F &\leq \|\bar{\vec{x}}\|\|\vec{1}_m\|\|\mat{E}\|_F \\
        &= \sqrt{m}\|\bar{\vec{x}}\|\|\mat{E}\|_F 
    \end{align*}
    and
    \begin{equation*}
        \|\mat{E}^\intercal \mat{E}\|_F \leq \|\mat{E}\|_F^2.
    \end{equation*}
    Combining these bounds, we have
    \begin{align*}
        \|\circled{A}\|_F &= \|\mat{B} + \mat{B}^\intercal + \mat{E}^\intercal \mat{E}\|_F \\
        &\leq 2\sqrt{m}\|\bar{\vec{x}}\|\|\mat{E}\|_F + \|\mat{E}\|_F^2.
    \end{align*}
    For $m\bar{\vec{x}}\bar{\vec{x}}^\intercal$ to dominate, we need 
    \begin{equation*}
        2\sqrt{m}\|\bar{\vec{x}}\|\|\mat{E}\|_F + \|\mat{E}\|_F^2 \ll m\|\bar{\vec{x}}\|^2.
    \end{equation*}
    For some small $\epsilon > 0$, this can be rewritten as
    \begin{equation*}
        2\sqrt{m}\|\bar{\vec{x}}\|\|\mat{E}\|_F + \|\mat{E}\|_F^2 \leq \epsilon m\|\bar{\vec{x}}\|^2.
    \end{equation*}
     Dividing both sides by $m\norm{\bar{\vec{x}}}^2$, we have
    \begin{equation*}
        \frac{2\|\mat{E}\|_F}{\sqrt{m}\|\bar{\vec{x}}\|} + \frac{\|\mat{E}\|_F^2}{m\|\bar{\vec{x}}\|^2} \leq \epsilon.
    \end{equation*}
    Letting $c = \frac{\|\mat{E}\|_F}{\sqrt{m}\|\bar{\vec{x}}\|} > 0$, we have the following inequality:
    \begin{equation*}
        c \leq \sqrt{1+\epsilon} -1.
    \end{equation*}
    This indicates that for a small $\epsilon > 0$, a sufficient condition for $\mat{X}^\intercal \mat{X} \approx m\bar{\vec{x}}\bar{\vec{x}}^\intercal$ to hold is
    \begin{equation*}
        \|\mat{E}\|_F \leq \sqrt{m}\norm{\bar{\vec{x}}}(\sqrt{1+\epsilon} -1).
    \end{equation*}
\end{proof}

\section{Proof of Theorem~\ref{thm:converge}}
\label{apdx:converge_proof}

\subsection{Notations}
We denote row-wise and column-wise Khatri-Rao products by $\ast$ and $\star$, respectively. The largest and smallest eigenvalues of a square matrix $\mat{X}$ are represented by $\lambda_{\max}(\mat{X})$ and $\lambda_{\min}(\mat{X})$, while the largest and smallest singular values are denoted as $\sigma_{\max}(\mat{X})$ and $\sigma_{\min}(\mat{X})$. 

\subsection{Required Lemmas}
First, we present the necessary lemmas to demonstrate the convergence of \mac.

\begin{lemma}[\cite{SchurBemerkungenZT}] \label{lem:1}
    For two positive definite matrices $\mat{A}$ and $\mat{B}$, we have
    \begin{align*}
        \lambda_{\max}(\mat{A} \odot \mat{B}) &\leq \max_{i}\mat{A}_{ii}\lambda_{\max}(\mat{B}) \\
        \lambda_{\min}(\mat{A} \odot \mat{B}) &\geq \min_{i}\mat{A}_{ii}\lambda_{\min}(\mat{B})
    \end{align*}
\end{lemma}

\begin{lemma}[Appendix D.3 in \cite{zhang2019fast}]
    $\mat{\Sigma}^{\infty}$ is strictly positive definite with minimum eigenvalue $\lambda_{\mat{\Gamma}}$, i.e., $\lambda_{\min}(\mat{\Gamma}^\infty(\mat{\Gamma}^\infty)^\intercal) = \lambda_{\mat{\Gamma}}$.
\end{lemma}

\begin{lemma}
    If $m=\Omega\left(\frac{n^2}{\lambda_{\mat{\Gamma}}}\log\left(\frac{2n^2}{\delta}\right)\right)$, we have with probability at least $1-\delta$
    \begin{equation*}
        \|\mat{\Sigma}(0)-\mat{\Sigma}^{\infty}\|_2 \leq \frac{\lambda_{\mat{\Gamma}}}{2} \quad \text{and} \quad \lambda_{\min}(\mat{\Sigma}(0)) \geq \frac{\lambda_{\mat{\Gamma}}}{2}.
    \end{equation*}
\end{lemma}

\begin{proof}
    We follow the same strategy as in Lemma 3.1 of \cite{du2018gradient}. Note that $\mat{\Sigma}_{ij}^{\infty}$ is an expectation of $\mat{\Sigma}(0)$ and hence by Hoeffding’s inequality we have with probability at least $1-\delta$
    \begin{align*}
        \mathbb{P}\left[\exists(i,j), \left| \mat{\Gamma}_{ij}(0) - \mat{\Gamma}_{ij}^{\infty}\right| \geq t \right] &= \mathbb{P}\left[\underset{(i,j)}{\bigcup} \left| \mat{\Gamma}_{ij}(0) - \mat{\Gamma}_{ij}^{\infty}\right| \geq t \right] \\
        &\leq \underset{(i,j)}{\sum} \mathbb{P}\left[ \left| \mat{\Gamma}_{ij}(0) - \mat{\Gamma}_{ij}^{\infty}\right| \geq t \right] \\
        &\leq 2n^2\exp{(-2mt^2)}
    \end{align*}
    Setting the above equal to $\delta$ and solving for time $t$ gives $t=\frac{\sqrt{\log{(2n^2 / \delta)}}}{\sqrt{2m}}$. Hence, for all $(i,j)$ with probability at least $1-\delta$, we have
    \begin{equation*}
        \left| \mat{\Gamma}_{ij}(0) - \mat{\Gamma}_{ij}^{\infty}\right| \leq \frac{\sqrt{\log{2n^2 / \delta}}}{\sqrt{2m}}.
    \end{equation*}
    Therefore, we have
    \begin{align*}
        \|\mat{\Sigma}(0)-\mat{\Sigma}^{\infty}\|_2 &\leq \|\mat{\Sigma}(0)-\mat{\Sigma}^{\infty}\|_{F} \\
        &\leq \underset{(i,j)}{\sum}\left| \mat{\Gamma}_{ij}(0) - \mat{\Gamma}_{ij}^{\infty}\right|^2 \\
        &\leq \frac{n^2\log{(2n^2 / \delta)}}{2m}
    \end{align*}
    Thus, if $m=\Omega\left(\frac{n^2}{\lambda_{\mat{\Gamma}}}\log\left(\frac{2n^2}{\delta}\right)\right)$, we have the desired result. Since $\lambda_{\min}(\mat{\Sigma}(0)) = \lambda_{\min}(\mat{\Sigma}^{\infty} + \mat{(\Sigma}(0)-\mat{\Sigma}^{\infty}))$, we have $\lambda_{\min}(\mat{\Sigma}(0)) \geq \frac{\lambda_{\mat{\Gamma}}}{2}$.
\end{proof}

\begin{lemma}[Lemma 3.2 in \cite{du2018gradient}]
    If weights are initialized i.i.d. from $\mathcal{N}(\vec{0}, \mat{I})$ with probability at least $1-\delta$, for $\vec{w}_1, \dots, \vec{w}_m \in \R^d$ that satisfy $\|\vec{w}_r(0) - \vec{w}_r\|_2 \leq \frac{c\delta\lambda_{\mat{\Gamma}}}{n^2} \triangleq R$ for some $c > 0$, the matrix $\mat{\Sigma}\in \R^{n\times n}$ defined by 
    \begin{equation*}
        \mat{\Sigma}_{ij} = \frac{1}{m}\sum_{r=1}^{m}\mathbb{I}(\vec{w}_{r}^\intercal \vec{x}_i \geq 0, \vec{w}_{r}^\intercal \vec{x}_j \geq 0)
    \end{equation*}
    satisfies $\|\mat{\Sigma} - \mat{\Sigma}(0)\|_2 < \frac{\lambda_{\mat{\Gamma}}}{4}$ and $\lambda_{\min}(\mat{\Sigma}) > \frac{\lambda_{\mat{\Gamma}}}{2}$.
\end{lemma}

While the matrix $\mat{\Sigma}$ is slightly different from the neural tangent kernel in \cite{du2018gradient}, the same proof applies.

\begin{lemma}
    For all $\vec{\theta}$ such that $\|\vec{\theta} - \vec{\theta}_0\|_2 \leq R$, we have with probability at least $1-\delta$
    \begin{equation}
        \|\mat{J} - \mat{J}(\vec{\theta}_0)\|_{2}^2 \leq \frac{2nR^{2/3}}{\delta^{2/3}m^{1/3}}. \label{eq:jacob_upperbound}
    \end{equation}
\end{lemma}

Setting $R = \frac{\sqrt{\lambda_{\max}(\mat{X}^\intercal \mat{X})\|\vec{y}-\vec{u}(\vec{\theta}_0)\|}}{\lambda_{\mat{\Gamma}}}$ in Eq. (\ref{eq:jacob_upperbound}) gives
\begin{equation*}
    \|\mat{J} - \mat{J}(\vec{\theta}_0)\|_{2}^2 \leq \frac{2n\lambda_{\max}(\mat{X}^\intercal \mat{X})^{1/3}\|\vec{y}-\vec{u}(\vec{\theta}_0)\|_{2}^{2/3}}{\lambda_{\mat{\Gamma}}^{2/3}\delta^{2/3}m^{1/3}}.
\end{equation*}
Therefore, choosing $m = \Omega\left(\frac{n^3\lambda_{\max}(\mat{X}^\intercal \mat{X})^{4}\|\vec{y}-\vec{u}(\vec{\theta}_0)\|_{2}^{2}}{\lambda_{\mat{\Gamma}}^{2}\delta^{2}\rho^3}\right)$, we can show that Condition~\ref{cond:2} is satisfied.

\subsection{Convergence of \mac}
Now, we are ready to prove Theorem~\ref{thm:converge}. We have
\begin{align*}
    u_i &= \frac{1}{\sqrt{m}}\vec{q}^\intercal \vec{x}_i \\
    \mat{J}_i &= \frac{\partial u_i}{\partial \vec{a}_i}\frac{\partial \vec{a}_i}{\partial \vec{z}_i}\frac{\partial \vec{z}_i}{\partial \vec{\theta}} = \vec{q}^\intercal \diag(\phi'(\vec{z}_i))(\vec{x}_i^\intercal \otimes \mat{I}_m) = \vec{x}_i^\intercal \ast \vec{s}_i^\intercal \\
    \vec{s}_i &= \frac{1}{\sqrt{m}}[q_1 \phi'(\vec{z}_i[1]), q_2 \phi'(\vec{z}_i[2]), \dots, q_m \phi'(\vec{z}_i[m])]^\intercal \\
    \mat{S} &= [\vec{s}_1, \dots, \vec{s}_n]^\intercal \in \R^{n\times m} \\
    \mat{J} &= \frac{\partial \vec{u}}{\partial \vec{\theta}} = [\mat{J}_1, \dots, \mat{J}_n]^\intercal = \mat{X} \ast \mat{S} \in \R^{n\times md}
\end{align*}

\begin{proof}
The change in the output between two consecutive steps is
\begin{align}
    \vec{u}(\vec{\theta}_{k+1}) &- \vec{u}(\vec{\theta}_{k})\nonumber \\
    &= \int_{s=0}^{1} \frac{\partial \vec{u}(\vec{\theta}_{k}^{k+1}(s))}{\partial \vec{\theta}} \cdot (\vec{\theta}_{k+1} - \vec{\theta}_{k}) \ ds \nonumber \\
    &= -\int_{s=0}^{1} \frac{\partial \vec{u}(\vec{\theta}_{k}^{k+1}(s))}{\partial \vec{\theta}} \cdot \eta \mat{F}_{\mac}^{-1} \nabla_{\vec{\theta}_k}\mathcal{L} \ ds \nonumber \\  
    &= \underbracket{-\eta\int_{s=0}^{1} \frac{\partial \vec{u}(\vec{\theta}_k)}{\partial \vec{\theta}} \cdot \mat{F}_{\mac}^{-1}\mat{J}(\vec{\theta}_{k})^\intercal(\vec{u}(\vec{\theta}_{k})-\vec{y}) \ ds}_{\circled{A}} \nonumber \\
    &\quad + \underbracket{\eta\int_{s=0}^{1}  \left(\frac{\partial \vec{u}(\vec{\theta}_k)}{\partial \vec{\theta}} - \frac{\partial \vec{u}(\vec{\theta}_{k}^{k+1}(s))}{\partial \vec{\theta}}\right) \cdot \mat{F}_{\mac}^{-1}\mat{J}(\vec{\theta}_{k})^\intercal(\vec{u}(\vec{\theta}_{k})-\vec{y}) \ ds}_{\circled{B}}. \label{eq:uk+1-uk}
\end{align}
where $\vec{\theta}_{k}^{k+1}(s) = \vec{\theta}_{k} + s(\vec{\theta}_{k+1}-\vec{\theta}_{k})$ is the convex combination of two consecutive parameter vectors with $s\in [0,1]$. With the definition of $\mat{F}_{\mac}$, the term $\circled{A}$ becomes
\begin{align*}
    \circled{A} &= \eta (\mat{X}\ast\mat{S})\left((\bar{\vec{x}}\bar{\vec{x}}^\intercal + \rho\mat{I}_d)^{-1} \otimes \mat{I}_m\right)(\mat{X}^\intercal\star\mat{S}^\intercal)(\vec{y}-\vec{u}(\vec{\theta}_{k})) \\
    &= \eta (\mat{X}\ast\mat{S})\left((\bar{\vec{x}}\bar{\vec{x}}^\intercal + \rho\mat{I}_d)^{-1}\mat{X}^\intercal \star \mat{S}^\intercal\right)(\vec{y}-\vec{u}(\vec{\theta}_{k})) \\
    &= \eta \left(\mat{X}(\bar{\vec{x}}\bar{\vec{x}}^\intercal + \rho\mat{I}_d)^{-1}\mat{X}^\intercal \odot \mat{S}\mat{S}^\intercal\right)(\vec{y}-\vec{u}(\vec{\theta}_{k})) \\
    &= \eta \left(\mat{X}(\bar{\vec{x}}\bar{\vec{x}}^\intercal + \rho\mat{I}_d)^{-1}\mat{X}^\intercal \odot \mat{\Gamma}\mat{\Gamma}^\intercal\right)(\vec{y}-\vec{u}(\vec{\theta}_{k})).
\end{align*}
where $\mat{\Gamma} =\frac{1}{\sqrt{m}}[\phi'(\mat{X}\vec{w}_1), \dots, \phi'(\mat{X}\vec{w}_m)]$. The second and third equality is derived using the following properties:
\begin{align}
    (\mat{A}\otimes \mat{B})(\mat{C}\star \mat{D}) &= (\mat{A}\mat{C})\star (\mat{B}\mat{D}) \nonumber \\
    (\mat{A}\ast \mat{B})(\mat{C}\star \mat{D}) &= (\mat{A}\mat{C})\odot (\mat{B}\mat{D}) \nonumber \\
\end{align}
The last equality holds due to the fact that $q_r \in (-1, 1)$ and hence they disappear in $\mat{S}\mat{S}^\intercal$.

Using Condition~\ref{cond:2}, we have
\begin{align*}
    \|\circled{B}\|_2 &= \eta\left\|\int_{s=0}^{1}\mat{J}(\vec{\theta}_k^{k+1}(s))-\mat{J}(\vec{\theta}_k) \ ds\right\|_2 \left\| (\bar{\vec{x}}\bar{\vec{x}}^\intercal + \rho\mat{I}_d)^{-1}\mat{X}^\intercal \star \mat{S}^\intercal \right\|_2 \left\|\vec{y}-\vec{u}(\vec{\theta}_{k}) \right\|_2 \\
    &\leq \eta \frac{C\rho}{\sigma_{\max}(\mat{X})} \underbracket{\left\| (\bar{\vec{x}}\bar{\vec{x}}^\intercal + \rho\mat{I}_d)^{-1}\mat{X}^\intercal \star \mat{S}^\intercal \right\|_2}_{\circled{C}}\left\|\vec{y}-\vec{u}(\vec{\theta}_{k}) \right\|_2.
\end{align*}
We bound the norm of $\circled{C}$.
\begin{align*}
    \sigma_{\max}(\circled{C}) &= \sqrt{\lambda_{\max}(\circled{C}^\intercal \circled{C})} \\
    &= \sqrt{\lambda_{\max}\left(\mat{X}(\bar{\vec{x}}\bar{\vec{x}}^\intercal + \rho\mat{I}_d)^{-1}(\bar{\vec{x}}\bar{\vec{x}}^\intercal + \rho\mat{I}_d)^{-1}\mat{X}^\intercal \odot \mat{S}\mat{S}^\intercal\right)} \\
    &\leq \sqrt{\frac{\lambda_{\max}(\mat{X}^\intercal \mat{X})}{\lambda_{\min}(\bar{\vec{x}}\bar{\vec{x}}^\intercal + \rho\mat{I}_d)^2}}\\
    &\leq \frac{\sigma_{\max}(\mat{X})}{\rho}.
\end{align*}
The first inequality holds by Lemma~\ref{lem:1}. The above bound show that we can increase the width $m$ of network to have $\circled{A} \gg \circled{B}$, allowing us to safely ignore $\circled{B}$ in Eq. (\ref{eq:uk+1-uk}) in the following analysis.

\begin{align*}
    \| \vec{y} &- \vec{u}(\vec{\theta}_{k+1})\|^2 \\
    &=  \| \vec{y} - \vec{u}(\vec{\theta}_{k}) - (\vec{u}(\vec{\theta}_{k+1}) - \vec{u}(\vec{\theta}_{k}))\|^2 \\
    &= \| \vec{y} - \vec{u}(\vec{\theta}_{k}) \|^2 - 2(\vec{y} - \vec{u}(\vec{\theta}_{k}))^\intercal (\vec{u}(\vec{\theta}_{k+1}) - \vec{u}(\vec{\theta}_{k})) + \|\vec{u}(\vec{\theta}_{k+1}) - \vec{u}(\vec{\theta}_{k})\|^2 \\
    &= \| \vec{y} - \vec{u}(\vec{\theta}_{k}) \|^2 - 2\eta (\vec{y} - \vec{u}(\vec{\theta}_{k}))^\intercal  \left(\mat{X}(\bar{\vec{x}}\bar{\vec{x}}^\intercal + \rho\mat{I}_d)^{-1}\mat{X}^\intercal \odot \mat{\Gamma}\mat{\Gamma}^\intercal\right)(\vec{y}-\vec{u}(\vec{\theta}_{k})) \\
    & \quad\quad + \eta^2 (\vec{y} - \vec{u}(\vec{\theta}_{k}))^\intercal  \left(\mat{X}(\bar{\vec{x}}\bar{\vec{x}}^\intercal + \rho\mat{I})^{-1}\mat{X}^\intercal \odot \mat{\Gamma}\mat{\Gamma}^\intercal\right)^2(\vec{y}-\vec{u}(\vec{\theta}_{k})) \\
    &\leq \| \vec{y} - \vec{u}(\vec{\theta}_{k}) \|^2 - \eta (\vec{y} - \vec{u}(\vec{\theta}_{k}))^\intercal  \left(\mat{X}(\bar{\vec{x}}\bar{\vec{x}}^\intercal + \rho\mat{I}_d)^{-1}\mat{X}^\intercal \odot \mat{\Gamma}\mat{\Gamma}^\intercal\right)(\vec{y}-\vec{u}(\vec{\theta}_{k})) \\
    &= \left(1 - \frac{\eta\lambda_{\mat{\Gamma}}\lambda_{\min}(\mat{X}^\intercal \mat{X})}{2\|\bar{\vec{x}}\|^{2} + \rho}\right) \| \vec{y} - \vec{u}(\vec{\theta}_{k}) \|^2.
\end{align*}
The first inequality holds if we set $\eta=\mathcal{O}\left(\frac{\rho}{\lambda_{\max}(\mat{X}^\intercal \mat{X})\lambda_{\mat{\Gamma}}}\right)$ and the last inequality holds due to the lower bound
\begin{align*}
    \lambda_{\min}\left(\mat{X}(\bar{\vec{x}}\bar{\vec{x}}^\intercal + \rho\mat{I}_d)^{-1}\mat{X}^\intercal \odot \mat{\Gamma}\mat{\Gamma}^\intercal\right) &\geq \lambda_{\min}\left(\mat{X}(\bar{\vec{x}}\bar{\vec{x}}^\intercal + \rho\mat{I}_d)^{-1}\mat{X}^\intercal\right)\lambda_{\min}\left(\mat{\Gamma}\mat{\Gamma}^\intercal\right) \\
    &\geq \frac{\lambda_{\min}(\mat{X}^\intercal\mat{X})}{\rho}\left(1 - \frac{\|\bar{\vec{x}}\|^2}{\|\bar{\vec{x}}\|^2 + \rho}\right)\frac{\lambda_{\mat{\Gamma}}}{2}.
\end{align*}

Next, we show that the weights of the network remain close to the initialization point.

\begin{align*}
    \|\vec{\theta}_{k+1} - \vec{\theta}_{0}\| &= \left\|\sum_{t=0}^{k}\left(\vec{\theta}_{t+1} - \vec{\theta}_{t}\right)\right\|_2\\
    &\leq \sum_{t=0}^{k}\left\|\left(\vec{\theta}_{t+1} - \vec{\theta}_{t}\right)\right\|_2 \\
    &= \eta \sum_{t=0}^{k}\left\|\left((\bar{\vec{x}}\bar{\vec{x}}^\intercal + \rho\mat{I}_d)^{-1}\mat{X}^\intercal \star \mat{S}^\intercal\right)(\vec{y}-\vec{u}(\vec{\theta}_t))\right\|_2 \\
    &\leq \eta \sum_{t=0}^{k}\frac{\sigma_{\max}(\mat{X})}{\rho}\|\vec{y}-\vec{u}(\vec{\theta}_t)\|_2 \\
    &\leq \eta \sum_{t=0}^{k}\left(1 - \frac{\eta\lambda_{\mat{\Gamma}}\lambda_{\min}(\mat{X}^\intercal \mat{X})}{2(\|\bar{\vec{x}}\|^{2} + \rho)}\right)^{t/2} \frac{\sigma_{\max}(\mat{X})}{\rho}\|\vec{y}-\vec{u}(\vec{\theta}_0)\|_2\\
    &\leq \frac{\sqrt{\lambda_{\max}(\mat{X}^\intercal \mat{X})}}{\lambda_{\mat{\Gamma}}}\|\vec{y}-\vec{u}(\vec{\theta}_0)\|_2.
\end{align*}
\end{proof}

\section{Experimental Details}
\label{apdx:exp_detail}

\subsection{CIFAR Training}
\label{apdx:cifar_detail}

Table~\ref{tab:cifar_hyperpar} provides the hyperparameter settings used in CIFAR-10/100 training. $\eta$ denotes the initial learning rate, 
$\rho$ denotes the damping factor, $\tau_{cov}$ is the update frequency for the curvature information matrix, and $\tau_{inv}$ represents the update frequency for the inverse. 
For KFAC and Eva, we used the recommended values as described in \cite{Zhang2023EvaPS}. 
For FOOF, LNGD, and \mac, we grid-searched the learning rate in the range of [0.001, 0.003, 0.01, 0.03, 0.1, 0.3] and the damping factor in the range of [0.001, 0.003, 0.01, 0.03, 0.1, 0.3, 1, 3], setting other hyperparameter values the same as for other KFAC variants.

\begin{table}[h]
    \caption{Hyperparameter values used in CIFAR datasets training}
    \label{tab:cifar_hyperpar}
    \centering
    \small
    \begin{tabular}{c|ccccccccc}
        \toprule
         & $\eta$ & momentum & EMA & weight decay & eps & $\rho$ & $\tau_{cov}$ & $\tau_{inv}$ \\ 
         \midrule
         \textsc{SGD} & 0.1 & 0.9 &  . &  0.0005 & . & . & . & .\\
         \textsc{AdamW} & 0.001 & . & (0.9, 0.999) & 0.5 & $10^{-8}$ & . & . & .\\
         \textsc{KFAC} & 0.1 & 0.9 & 0.95 & 0.0005 & . & 0.03 & 5 & 50\\
         \textsc{FOOF} & 0.1 & 0.9 & 0.95 & 0.0005 & . & 1.0 & 5 & 50\\
         \textsc{Eva} & 0.1 & 0.9 & 0.95 & 0.0005 & . & 0.03 & 5 & 50\\
         \textsc{LNGD} & 0.1 & 0.9 & 0.95 &  0.0005 & . & 1.0 & 5 & 50\\
         \mac & 0.1 & 0.9 & 0.95 &  0.0005 & . & 1.0 & 5 & 50\\
         \bottomrule
    \end{tabular}
\end{table}

\subsection{ImageNet-1k Training}
\label{apdx:imagenet_detail}

Table~\ref{tab:imagenet_setting} provides the training settings for the ImageNet-1k dataset. We used two architectures: ResNets and ViTs (DeiT and Swin). For ResNets, we referenced the implementation from PyTorch~\cite{paszke2019pytorch}, while for ViTs, we used the same settings as described in \cite{Touvron2020TrainingDI}. The training resolution for ResNets was 176 and for ViTs was 224, while the test resolution for both was set to 224. Both architectures used a mini-batch size of 1,024 and employed cosine learning rate decay. A warmup period of 5 epochs was applied. For all model architectures, we employed
Random Erasing~\cite{Zhong2017RandomED,Devries2017ImprovedRO}, Label
Smoothing~\cite{Szegedy2015RethinkingTI},
Mixup~\cite{zhang2018mixup}/Cutmix~\cite{Yun2019CutMixRS}, and
Repeated Augmentation~\cite{Hoffer2019AugmentYB,Berman2019MultiGrainAU}. For ResNets, we used
TrivialAugment~\cite{Mller2021TrivialAugmentTY} while for DeiT and
Swin, we utilized RandAugment~\cite{Cubuk2019RandaugmentPA}. 
Specifically, we utilized label smoothing (0.1 for both)
and repeated augmentation. Data augmentation techniques included horizontal flip, random resized crop, auto augmentation (TrivialAugment for ResNets and RandAugment(9/0.5) for ViTs), mixup (0.2 for ResNets and 0.8 for ViTs)/cutmix (1.0 for both), and random erasing (0.1 for ResNets and 0.25 for ViTs).

\begin{table}[h]
    \caption{Training settings for ImageNet-1k dataset}
    \label{tab:imagenet_setting}
    \centering
    \small
    \begin{tabular}{c|cc}
        \toprule
        Architecture & ResNets & ViTs \\
        \midrule
        Reference & Pytorch & DeiT\\
        \midrule
        Train Res & 176 & 224\\
        Test Res & 224 & 224 \\
        \midrule
        Batch size & 1,024 & 1,024 \\
        LR decay & cosine & cosine\\
        Warmup epochs & 5 & 5\\
        \midrule
        Label Smoothing & 0.1 & 0.1 \\
        Repeated Augmentation & \checkmark & \checkmark \\
        \midrule
        Horizontal flip & \checkmark  & \checkmark  \\
        Random Resized Crop & \checkmark  & \checkmark  \\
        Auto Augmentation & TrivialAugment & RandAugment(9/0.5) \\
        Mixup & 0.2 & 0.8\\
        Cutmix & 1.0 & 1.0 \\
        Random Erasing & 0.1 & 0.25 \\
        \bottomrule
    \end{tabular}
\end{table}

Table~\ref{tab:imagenet_hyperpar} illustrates the hyperparameter settings used for training on the ImageNet-1k dataset. For a mini-batch size of 1,024, we increased the learning rate $\eta$ by 5 times for both SGD and KFAC variants compared to the CIFAR training settings. When training ViTs, we grid-searched the damping factor in the range of [0.01, 0.03, 0.1, 0.3, 1, 3, 10] and reduced the inversion frequency of the preconditioners from 50 to 5 for the KFAC variants. This adjustment was necessary because most of the other KFAC variants, except \mac, experienced training failures due to issues with computing the inverse of the preconditioners under the settings used for ResNet training, while \mac remained stable with these settings. Additionally, for \mac, we adopted the decoupled weight decay~\cite{Loshchilov2019DecoupledWD}.

\begin{table}[tb]
    \caption{Hyperparameter values used in ImageNet-1k datasets training (ResNets / ViTs)}
    \label{tab:imagenet_hyperpar}
    \centering
    \small
    \begin{tabular}{c|cccccccc}
        \toprule
         & $\eta$ & momentum & EMA & weight decay & eps & $\rho$ & $\tau_{cov}$ & $\tau_{inv}$ \\ 
         \midrule
         \textsc{SGD} & 0.5 & 0.9 & . &  0.00002 & . & . & . & .\\
         \textsc{AdamW} & 0.001 & . & (0.9, 0.999) & 0.05 & $10^{-8}$ & . & . \\
         \textsc{KFAC} & 0.5 & 0.9 &  0.95 & 0.00002 & . &  0.03 / 0.3 & 5 & 50 / 5\\
         \textsc{FOOF} & 0.5 & 0.9 & 0.95 & 0.00002 & . &  1 / 3 & 5 & 50 / 5\\
         \textsc{Eva} & 0.5 & 0.9 &  0.95 & 0.00002 & . &  0.03 / 0.3 & 5 & 50 / 5\\
         \mac & 0.5 & 0.9 &  0.95 &  0.00002 / 0.0001 & . &  1 / 3 & 5 & 50 / 5\\
         \bottomrule
    \end{tabular}
\end{table}

\begin{figure}[b]
    \centering
    \includegraphics[width=0.47\linewidth]{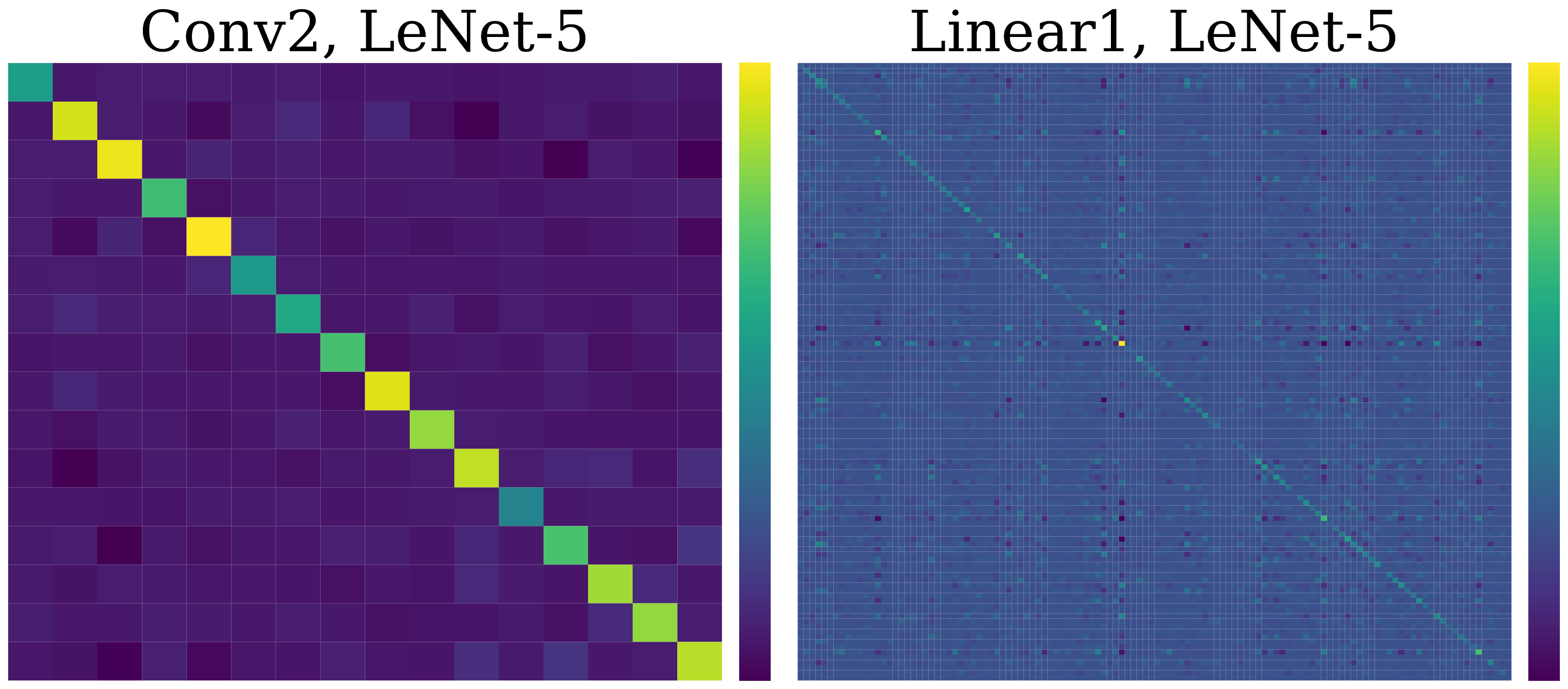}
    \includegraphics[width=0.483\linewidth]{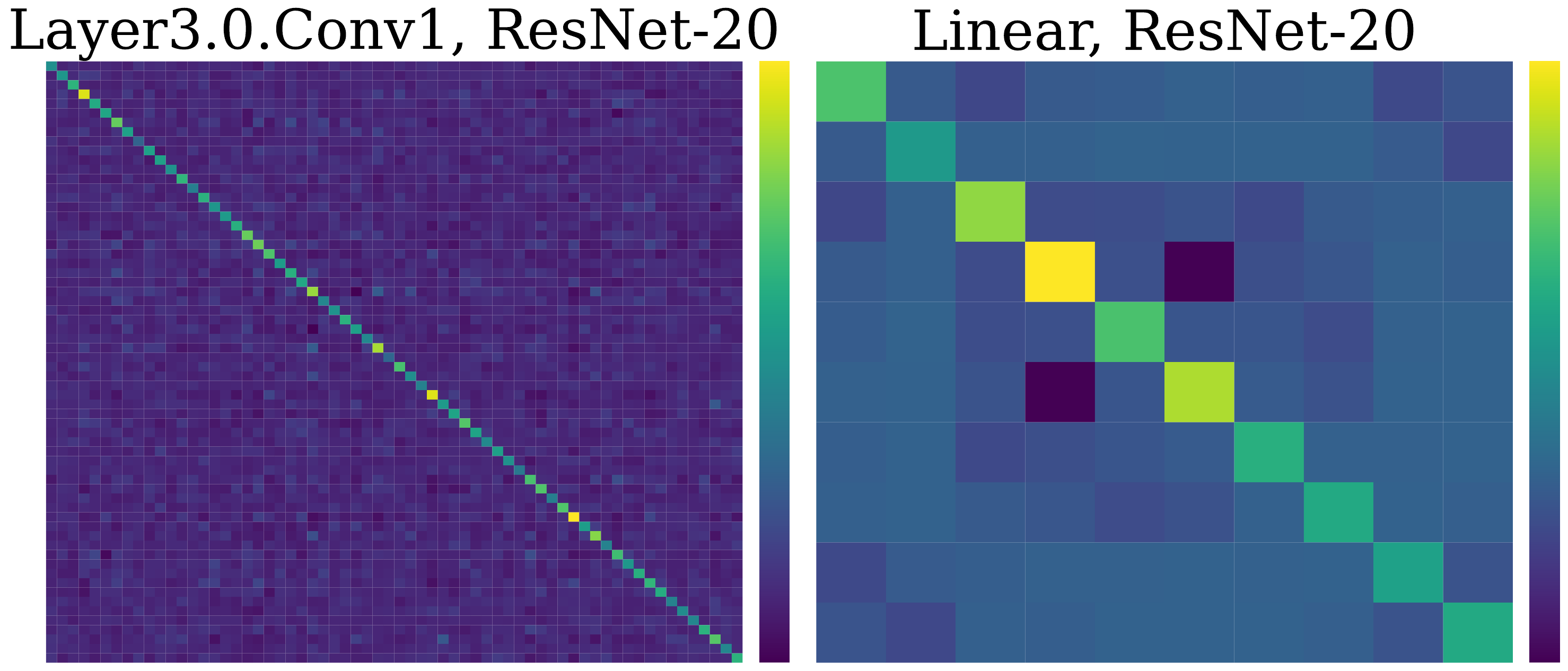}
    \caption{Heatmap visualizations of the pre-activation
      gradient KF for convolutional and linear layers in LeNet-5
      and ResNet-20.  
    } 
    \label{fig:p_heatmap}
\end{figure}

\section{Diagonal Approximation of Pre-activation Gradient KF}
As discussed in Section~\ref{sec:introduction} and Section~\ref{subsec:eigenspectrum}, the pre-activation gradient KF, $\mat{P}$, has little impact on the FIM's eigenspace because the magnitude of the eigenvalues of $\mat{P}$ is largely determined by those of activation KF, $\mat{A}$. Figure~\ref{fig:p_heatmap} further illustrates that the diagonal entries of $\mat{P}$ are significantly larger than its off-diagonal entries. While prior work such as Eva applies a rank-1 approximation to $\mat{P}$ and LNGD maintains per-example diagonal entries $\diag(\vec{p}_i\vec{p}_i^\intercal)$, where $i \in \mathcal{B}$, to approximate it, the most efficient approach is simply to utilize the diagonal entries of $\mat{P}$ directly. 

\subsection{Algorithm}
Based on these observations, we propose an additional method named \smac (pre-activation gradient \textsc{S}econd moment and \textsc{M}ean \textsc{A}ctivation approximated \textsc{C}urvature). In \smac, the activation factor $\mat{A}$ is approximated using the same rank-1 approximation as in \mac, while the pre-activation gradient factor $\mat{P}$ is approximated by a diagonal matrix.

\renewcommand{\algorithmiccomment}[1]{$\triangleright$ #1} 
\begin{algorithm}[t]
\caption{\smac}
\label{alg:smac_pseudocode}
\textbf{Require}: Learning rate $\eta_k$, Momentum $\beta_{1}$, EMA $\beta_{2}$,
Damping $\rho$, Curvature update frequency $\tau_{\text{cov}}$, Inverse update frequency $\tau_{\text{inv}}$\\
\textbf{Initialize}: $\vec{\theta}_{0}$, $\widetilde{\vec{a}}_0 = \vec{0}$, $\widetilde{\vec{p}}_0 = \vec{0}$, $\widehat{\mat{A}}^{-1} = \widehat{\mat{P}}^{-1} = \mat{I}$, $k_{\tau}=0$
\begin{algorithmic}[1]
\STATE (optional)  $(\mat{A}^{(1)})^{-1} = \left(\mat{I}
  - \frac{\bar{\vec{a}}^{(0)} (\bar{\vec{a}}^{(0)})^{\intercal}}{\rho +
    \norm{\bar{\vec{a}}^{(0)}}^{2}}\right)$, where $\bar{\vec{a}}^{(0)} =
\frac{1}{n}\sum_{i=1}^{n}\vec{x}_i$
\FOR{$k$ = 1, 2, 3, \dots}
\STATE $\mat{G}_{k} \gets \frac{1}{|\mathcal{B}|}\sum_{i\in
  \mathcal{B}}\grad{\ell}(f(\vec{x}_i; \vec{\theta}_k), y_i)$
\FOR{$l$ = 1, 2, \dots, L}
\IF{$(k \mod \tau_{\text{cov}}) = 0$}
\STATE $\widetilde{\vec{a}}_{k}^{(l)} \gets \beta_{2}\widetilde{\vec{a}}_{k-1}^{(l)} +
(1-\beta_{2})\bar{\vec{a}}_{k}^{(l-1)}$ 
\STATE $\widetilde{\vec{p}}_{k}^{(l)} \gets \beta_{2}\widetilde{\vec{p}}_{k-1}^{(l)} + (1-\beta_{2})(\vec{p}_{k}^{(l)})^2$ \label{alg:p_update}
\STATE $k_{\tau} \gets k_{\tau} + 1$
\ENDIF
\IF{$(k \mod \tau_{\text{inv}}) = 0$}
\STATE $\widehat{\vec{a}}_{k}^{(l)} \gets \widetilde{\vec{a}}_{k}^{(l)} \ / \
(1-\beta_{2}^{k_{\tau}})$
\STATE $\widehat{\vec{p}}_{k}^{(l)} \gets \widetilde{\vec{p}}_{k}^{(l)} \ / \ (1-\beta_{2}^{k_{\tau}})$ 
\STATE $(\widehat{\mat{A}}^{(l)})^{-1} \gets \left(\mat{I} - \frac{\widehat{\vec{a}}_{k}^{(l)}(\widehat{\vec{a}}_{k}^{(l)})^{\intercal}}{\rho + \|\widehat{\vec{a}}_{k}^{(l)}\|^{2}}\right)$  \hfill\COMMENT{decoupled damping in \eqref{eq:decoupled_damp}} 
\STATE $(\widehat{\mat{P}}^{(l)})^{-1} \gets \diag\left(\widehat{\vec{p}}_{k}^{(l)}+\rho\right)^{-1}$  \label{alg:p_inv}
\ENDIF
\STATE $\widehat{\mat{G}}_{k}^{(l)} \gets
(\widehat{\mat{P}}^{(l)})^{-1}\mat{G}_{k}^{(l)}(\widehat{\mat{A}}^{(l)})^{-1}$
\ENDFOR
\ENDFOR
\end{algorithmic}
\end{algorithm}
The pseudocode of \smac is presented in
Algorithm~\ref{alg:smac_pseudocode}. \smac additionally maintains the EMA of $(\vec{p}^{(l)})^{2} = \vec{p}^{(l)} \odot \vec{p}^{(l)}$ (Line~\ref{alg:p_update}). Since $\widehat{\mat{P}}^{(l)}$ is a diagonal matrix, its inverse is obtained by taking the reciprocal of its diagonal entries (Line~\ref{alg:p_inv}).

Compared to \mac, \smac incurs an additional $\mathcal{O}(d_{out})$ overhead to both time and memory complexities when computing the inverse of the approximated FIM, yet it remains considerably more scalable than other KFAC variants. 

\subsection{Experimental Results}
We evaluate \smac on CIFAR and ImageNet datasets under the same experimental conditions as \mac and compare its performance with other baselines, including \mac, used in Section~\ref{sec:experiment}.

\begin{table}[tb]
\caption{Test accuracy (\%) and standard deviation (in parentheses) on CIFAR-10 \textbf{(Top)} and CIFAR-100 \textbf{(Bottom)} datasets.}
\label{tab:smac_cifar}
\centering
\footnotesize
\begin{tabular}{c|ccc|ccc|ccc}
\toprule
Model & \multicolumn{3}{c|}{ResNet-110} & \multicolumn{3}{c|}{DenseNet-121} & \multicolumn{3}{c}{WideResNet-28-10} \\
Epoch & 50 & 100 & 200 & 50 & 100 & 200 & 50 & 100 & 200  \\
\midrule
\smac & 93.8\scriptsize{(0.2)} & 94.4\scriptsize{(0.2)} & 95.0\scriptsize{(0.1)} & 95.3\scriptsize{(0.2)} & 95.8\scriptsize{(0.1)} & 95.9\scriptsize{(0.1)} & 95.6\scriptsize{(0.1)} & 96.2\scriptsize{(0.2)} & 96.4\scriptsize{(0.2)} \\
\mac    & 93.5\scriptsize{(0.2)} & 94.4\scriptsize{(0.2)} & 94.9\scriptsize{(0.2)} & 95.3\scriptsize{(0.2)} & 95.7\scriptsize{(0.1)} & 95.9\scriptsize{(0.1)} & 95.7\scriptsize{(0.1)} & 96.2\scriptsize{(0.1)} & 96.4\scriptsize{(0.1)} \\
\midrule
\textsc{SGD}    & 92.0\scriptsize{(0.3)} & 93.5\scriptsize{(0.4)} & 94.3\scriptsize{(0.3)} & 94.9\scriptsize{(0.2)} & 95.4\scriptsize{(0.1)} & 95.6\scriptsize{(0.1)} & 95.4\scriptsize{(0.1)} & 96.0\scriptsize{(0.2)} & 96.2\scriptsize{(0.0)} \\
\textsc{AdamW}  & 93.4\scriptsize{(0.1)} & 94.2\scriptsize{(0.1)} & 94.4\scriptsize{(0.1)} & 94.7\scriptsize{(0.1)} & 94.9\scriptsize{(0.2)} & 94.9\scriptsize{(0.2)} & 95.1\scriptsize{(0.1)} & 95.6\scriptsize{(0.1)} & 95.9\scriptsize{(0.1)} \\
\midrule
\textsc{KFAC}   & 93.5\scriptsize{(0.1)} & 94.3\scriptsize{(0.1)} & 94.7\scriptsize{(0.2)} & 94.9\scriptsize{(0.1)} & 95.3\scriptsize{(0.1)} & 95.6\scriptsize{(0.1)} & 95.4\scriptsize{(0.1)} & 96.0\scriptsize{(0.1)} & 96.3\scriptsize{(0.1)}  \\
\textsc{FOOF}   & 94.0\scriptsize{(0.1)} & 94.7\scriptsize{(0.1)} & 95.1\scriptsize{(0.1)} & 95.6\scriptsize{(0.1)} & 95.8\scriptsize{(0.1)} & 96.0\scriptsize{(0.1)} & 95.8\scriptsize{(0.1)} & 96.2\scriptsize{(0.1)} & 96.4\scriptsize{(0.0)}  \\
\textsc{Eva}    & 93.6\scriptsize{(0.2)} & 94.1\scriptsize{(0.1)} & 94.7\scriptsize{(0.1)} & 94.7\scriptsize{(0.1)} & 95.3\scriptsize{(0.1)} & 95.7\scriptsize{(0.2)} & 95.4\scriptsize{(0.2)} & 95.9\scriptsize{(0.1)} & 96.2\scriptsize{(0.2)}  \\
\textsc{LNGD}   & 92.8\scriptsize{(0.1)} & 93.8\scriptsize{(0.2)} & 94.1\scriptsize{(0.1)} & 95.0\scriptsize{(0.1)} & 95.4\scriptsize{(0.2)} & 95.3\scriptsize{(0.2)} & 95.2\scriptsize{(0.2)} & 95.6\scriptsize{(0.1)} & 95.8\scriptsize{(0.2)}  \\
\bottomrule
\toprule
Model & \multicolumn{3}{c|}{ResNet-110} & \multicolumn{3}{c|}{DenseNet-121} & \multicolumn{3}{c}{WideResNet-28-10}  \\
Epoch & 50 & 100 & 200 & 50 & 100 & 200 & 50 & 100 & 200  \\
\midrule
\smac & 72.5\scriptsize{(0.3)} & 74.3\scriptsize{(0.2)} & 74.9\scriptsize{(0.5)} & 78.8\scriptsize{(0.2)} & 80.7\scriptsize{(0.2)} & 80.7\scriptsize{(0.3)} & 79.3\scriptsize{(0.1)} & 81.0\scriptsize{(0.2)} & 81.6\scriptsize{(0.3)} \\
\mac    & 72.8\scriptsize{(0.4)} & 74.2\scriptsize{(0.4)} & 75.0\scriptsize{(0.3)} & 78.9\scriptsize{(0.2)} & 80.5\scriptsize{(0.2)} & 80.6\scriptsize{(0.2)} & 79.4\scriptsize{(0.2)} & 80.9\scriptsize{(0.2)} & 81.6\scriptsize{(0.2)} \\
\midrule
\textsc{SGD}    & 71.1\scriptsize{(1.0)} & 72.5\scriptsize{(0.7)} & 73.5\scriptsize{(0.6)} & 78.2\scriptsize{(0.2)} & 79.6\scriptsize{(0.1)} & 79.9\scriptsize{(0.3)} & 79.3\scriptsize{(0.2)} & 80.7\scriptsize{(0.1)} & 81.5\scriptsize{(0.2)} \\
\textsc{AdamW}  & 71.6\scriptsize{(0.1)} & 73.4\scriptsize{(0.2)} & 73.7\scriptsize{(0.3)} & 77.3\scriptsize{(0.2)} & 78.5\scriptsize{(0.3)} & 79.0\scriptsize{(0.2)} & 78.2\scriptsize{(0.2)} & 79.7\scriptsize{(0.1)} & 80.2\scriptsize{(0.2)} \\
\midrule
\textsc{KFAC}   & 71.5\scriptsize{(1.3)} & 73.2\scriptsize{(0.5)} & 74.2\scriptsize{(0.4)} & 78.1\scriptsize{(0.5)} & 79.7\scriptsize{(0.2)} & 80.1\scriptsize{(0.3)} & 79.3\scriptsize{(0.2)} & 81.0\scriptsize{(0.2)} & 81.5\scriptsize{(0.1)} \\
\textsc{FOOF}   & 73.6\scriptsize{(0.2)} & 75.1\scriptsize{(0.3)} & 76.0\scriptsize{(0.4)} & 79.8\scriptsize{(0.1)} & 80.9\scriptsize{(0.2)} & 81.0\scriptsize{(0.2)} & 80.0\scriptsize{(0.2)} & 80.8\scriptsize{(0.2)} & 81.0\scriptsize{(0.3)}  \\
\textsc{Eva}    & 71.9\scriptsize{(0.6)} & 73.6\scriptsize{(0.3)} & 74.7\scriptsize{(0.3)} & 77.9\scriptsize{(0.3)} & 79.4\scriptsize{(0.3)} & 79.9\scriptsize{(0.2)} & 79.3\scriptsize{(0.3)} & 81.0\scriptsize{(0.2)} & 81.6\scriptsize{(0.3)}  \\
\textsc{LNGD}   & 71.7\scriptsize{(0.2)} & 73.1\scriptsize{(0.3)} & 74.5\scriptsize{(0.1)} & 78.8\scriptsize{(0.1)} & 79.9\scriptsize{(0.1)} & 79.7\scriptsize{(0.3)} & 79.1\scriptsize{(0.2)} & 79.8\scriptsize{(0.2)} & 79.4\scriptsize{(0.2)}  \\
\bottomrule
\end{tabular}
\end{table}

\begin{table}[tb]
\caption{Top-1 accuracy (\%) on ImageNet-1k for ResNet-50 and ResNet-101.}
\label{tab:smac_imagenet}
\centering
\small
\begin{tabular}{c|cc|cc}
\toprule
Model & \multicolumn{2}{c|}{ResNet-50} & \multicolumn{2}{c}{ResNet-101} \\
Epoch & 100 & 200 & 100 & 200 \\
\midrule
\smac         & 78.1 & 79.9 & 79.9 & 81.1 \\
\mac          & 78.0 & 79.7 & 79.8 & 81.2 \\
\midrule 
\textsc{SGD}    & 78.1 & 79.6 & 79.7 & 81.3 \\
\textsc{AdamW}  & 76.8 & 79.2 & 77.9 & 80.6 \\
\midrule
\textsc{KFAC}   & 78.2 & 79.3 & 79.6 & 81.1 \\
\textsc{FOOF}   & 78.4 & 79.7 & 80.0 & 81.0 \\
\textsc{Eva}    & 77.7 & 79.4 & 79.6 & 81.1 \\
\bottomrule
\end{tabular}
\end{table}

As shown in Table~\ref{tab:smac_cifar} and Table~\ref{tab:smac_imagenet}, \smac achieves test accuracies comparable to \mac across all datasets, models, and training epochs, while consistently outperforming competitive counterparts such as Eva and LNGD. These results indicate that \smac preserves the high accuracy of \mac and offers a more efficient approach compared to other second-order methods. Specifically, our experimental findings demonstrate that: \begin{enumerate*}[label=(\roman*)] \item Compared to Eva, \smac approximates $\mat{P}$ using $\diag(\E[\vec{p}^2])$ rather than $\E[\vec{p}]\E[\vec{p}]^\intercal$. This modification reduces computational costs and improves test accuracy, suggesting that leveraging the diagonal of the second moment provides a more efficient and scalable approximation of $\mat{P}$ than its outer product form. 
\item Compared to LNGD, \smac employs an aggregated second moment approximation $\diag(\E[\vec{p}^2])$ for $\mat{P}$ instead of using per-example second moments combined with the activation L2 norm $\frac{\E[\|\vec{a}\|^2 \cdot \text{diag}(\vec{p}\vec{p}^\intercal)]}{\E[\|\vec{a}\|^2\|\vec{p}\|^2]}$ as in LNGD. This change significantly reduces memory usage while simultaneously enhancing test accuracy, even though \smac approximates $\mat{A}$ with a rank-1 matrix formed from the outer product of the mean activation.
\end{enumerate*}

Meanwhile, although \smac achieves performance competitive with \mac on CNNs, its application to the self-attention mechanism in transformer architectures presents additional challenges. As discussed in Section~\ref{subsec:curve_trf}, the derivation of pre-activation gradient KFs in attention layers differs significantly from that in fully-connected layers and convolutional layers. This discrepancy necessitates the implementation of customized hook functions to capture $\partial \mathcal{L} / \partial \mat{R}$ and $\partial \mathcal{L} / \partial \mat{H}$ and additional matrix-vector multiplications, which further escalate the computational cost.


\end{document}